\documentclass[twoside,11pt,tablecaption=bottom]{jmlr}

\usepackage{times}

\usepackage{mathtools}
\usepackage{enumitem}
\newlist{inlistalph}{enumerate*}{1}
\setlist[inlistalph]{label=(\alph*)}

\usepackage[utf8]{inputenc}
\usepackage[english]{babel}
\usepackage{amsopn}
\usepackage{amsfonts}
\usepackage{mathrsfs}
\usepackage{xparse}
\usepackage{upgreek}

\usepackage{wasysym}
\usepackage{color}
\usepackage{tabu}

\usepackage{float}

\usepackage{verbatim}

\usepackage{adjustbox}

\newcommand*{\qedhere}{\hfill\BlackBox\\[2mm]}
\newcommand*{\claimqedhere}{\hfill (Claim) \BlackBox\\[2mm]}

\newcommand*{\Rmonoid}{$\totalCp$-monoid}

\usepackage{thm-restate}
\newcounter{count}
\newtheorem{claim}[count]{Claim}

\newcommand*{\falls}{\text{if }}
\newcommand*{\sonst}{\text{otherwise.}}
\newcommand*{\sonstfalls}{\text{else, if }}

\usepackage{tikz}
\usepackage{xspace}
\usetikzlibrary{shapes, fit, decorations.pathreplacing, shadows, arrows, calc, positioning}

\usepackage{pgfplots}
\usepgfplotslibrary{patchplots}

\newcommand*{\Txt}{\mathbf{Txt}}

\newcommand*{\G}{\mathbf{G}}

\newcommand*{\Sd}{\mathbf{Sd}}
\newcommand*{\Psd}{\mathbf{Psd}}

\newcommand*{\Ex}{\mathbf{Ex}}
\newcommand*{\Bc}{\mathbf{Bc}}

\newcommand*{\Fin}{\mathbf{Fin}}

\newcommand*{\Cons}{\mathbf{Cons}}
\newcommand*{\Conv}{\mathbf{Conv}}
\newcommand*{\Caut}{\mathbf{Caut}}
\newcommand*{\CautTar}{\Caut_{\textup{\textbf{Tar}}}}

\newcommand*{\Wb}{\mathbf{Wb}}

\newcommand*{\WMon}{\mathbf{WMon}}

\newcommand*{\Sem}{\mathbf{Sem}}

\newcommand*{\SemWb}{\Sem\Wb}
\newcommand*{\SemConv}{\Sem\Conv}

\newcommand*{\T}{\mathbf{T}}

\newcommand*{\N}{\mathbb{N}}

\newcommand*{\Pow}{\mathrm{Pow}}

\newcommand*{\La}{\mathcal{L}}
\newcommand*{\Ia}{\mathcal{I}}
\newcommand*{\Sa}{\mathcal{S}}

\newcommand*{\totalCp}{\mathcal{R}}
\newcommand*{\partialCp}{\mathcal{P}}

\newcommand{\dom}{\mathrm{dom}}
\newcommand{\range}{\mathrm{range}}
\newcommand{\content}{\mathrm{content}}

\newcommand{\ind}{\mathrm{ind}}
\newcommand{\pad}{\mathrm{pad}}
\newcommand{\unpad}{\mathrm{unpad}}

\newcommand{\refuted}{\mathrm{refuted}}

\newcommand{\ORT}{\textbf{ORT}\xspace}

\makeatletter
\newsavebox{\@brx}
\newcommand{\llangle}[1][]{\savebox{\@brx}{\(\m@th{#1\langle}\)}%
  \mathopen{\copy\@brx\kern-0.5\wd\@brx\usebox{\@brx}}}
\newcommand{\rrangle}[1][]{\savebox{\@brx}{\(\m@th{#1\rangle}\)}%
  \mathclose{\copy\@brx\kern-0.5\wd\@brx\usebox{\@brx}}}
\makeatother

\newcommand{\divs}{\mathclose{\hbox{$\uparrow$}}}

\newcommand*{\concat}{^\frown}

\newcommand*{\noqed}{\renewcommand{\jmlrQED}{}}

\newcommand*{\Sq}{\mathbb{S}}

\makeatletter 
\g@addto@macro{\@algocf@init}{\SetKwInput{input}{Input}} 
\makeatother
\makeatletter 
\g@addto@macro{\@algocf@init}{\SetKwInput{outoutput}{Output}} 
\makeatother
\makeatletter 
\g@addto@macro{\@algocf@init}{\SetKwInput{param}{Parameter}} 
\makeatother
\makeatletter 
\g@addto@macro{\@algocf@init}{\SetKwInput{output}{Semantic Output}} 
\makeatother
\makeatletter 
\g@addto@macro{\@algocf@init}{\SetKwInput{init}{Initialization}} 
\makeatother

\newcommand{\itemin}[1]{\item[#1\hspace{-0.5cm}] \hspace{0.5cm}}

\pgfdeclarelayer{background}
\pgfdeclarelayer{foreground}
\pgfsetlayers{background,main,foreground}

\jmlrvolume{}
\jmlryear{}
\jmlrsubmitted{}
\jmlrpublished{}
\jmlrworkshop{}
\jmlrpages{}
\jmlrproceedings{}{}

\title[Normal Forms for (Semantically) Witness-Based Learners in Inductive Inference]{Normal Forms for (Semantically) Witness-Based Learners \\ in Inductive Inference}

\author{\Name{Vanja Dosko\v{c}} \Email{vanja.doskoc@hpi.de} \\ \Name{Timo K\"{o}tzing} \Email{timo.koetzing@hpi.de} \\ \addr Hasso Plattner Institute \\ University of Potsdam, Germany}

\begin{document}

\maketitle

\begin{abstract}
  We study learners (computable devices) inferring formal languages, a setting referred to as \emph{language learning in the limit} or \emph{inductive inference}. In particular, we require the learners we investigate to be \emph{witness-based}, that is, to \emph{justify} each of their mind changes. Besides being a natural requirement for a learning task, this restriction deserves special attention as it is a specialization of various important learning paradigms. In particular, with the help of witness-based learning, \emph{explanatory} learners are shown to be equally powerful under these seemingly incomparable paradigms. Nonetheless, until now, witness-based learners have only been studied sparsely. 

  In this work, we conduct a thorough study of these learners both when requiring \emph{syntactic} and \emph{semantic} convergence and obtain \emph{normal forms} thereof. In the former setting, we extend known results such that they include witness-based learning and generalize these to hold for a variety of learners. Transitioning to \emph{behaviourally correct} learning, we also provide normal forms for \emph{semantically witness-based} learners. Most notably, we show that \emph{set-driven globally} semantically witness-based learners are equally powerful as their \emph{Gold-style semantically conservative} counterpart. Such results are key to understanding the, yet undiscovered, mutual relation between various important learning paradigms when learning behaviourally correctly.
\end{abstract}


\section{Introduction}

In his seminal paper, \citet{Gold67} studied the algorithmic learning of formal languages from a growing but finite amount of information thereof. This marked the starting point of \emph{inductive inference} or \emph{language learning in the limit}, a branch of (algorithmic) learning theory. Here, a learner $h$ (a computable device) is successively presented all and only the information from a formal language $L$ (a computably enumerable subset of the natural numbers). We call such a list of elements of $L$ a \emph{text} of $L$. When given a new datum, the learner $h$ makes a guess (a description for a computably enumerable set) about which language it believes to be presented. Once these guesses converge to a single, correct hypothesis explaining the language, the learner successfully \emph{learned} the language $L$ on this text. We say that $h$ \emph{learns} $L$, if it learns $L$ on every text of $L$.

We refer to this as \emph{explanatory learning} as the learner, in the limit, provides an explanation of the presented language and denote it as $\Txt\G\Ex$. Here, $\Txt$ indicates that the information is given from text, $\G$ stands for \emph{Gold-style} learning, where the learner has \emph{full information} on the elements presented to make its guess, and, lastly, $\Ex$ refers to explanatory learning. Since a learner which always guesses a particular language can learn this very language, we study classes of languages which can be $\Txt\G\Ex$-learned by a single learner and denote the set of all such classes with $[\Txt\G\Ex]$. We refer to this set as the \emph{learning power} of $\Txt\G\Ex$-learners.

Many additional restrictions may be imposed on the learners. For example, we may apply memory restrictions, change the criterion for successful learning or require the learner to refrain from certain unwanted behaviour. In this paper, we require each learner to \emph{justify} any mind change it makes. \citet{KS16} introduced this as \emph{witness-based} learning ($\Wb$) as a means to specialize many important restrictions in inductive inference, including \emph{conservativeness} \citep{Angluin80}, \emph{weak monotonicity} \citep{Jantke91,Wiehagen91} and \emph{cautiousness} \citep{OSW82}, see Section~\ref{Sec:Prelim} for detailed definitions. While conservative learners are always weakly monotone, there is no general connection to cautious learners. Nonetheless, often these learners are equally powerful, see \citet{KS95} or \citet{KP16}.

A key result of \citet{KS16} reveals an explanation for this phenomenon. In the setting they study, they show that witness-based learners are equally powerful as \emph{target-cautious} learners ($\CautTar$), which may never overgeneralise the target language \citep{KP16}. Note that this restriction is a generalization of all the restrictions mentioned above. In particular, they study \emph{partially set-driven} learners ($\Psd$), see \citet{BlumBlum75} and \citet{SchRicht84}, which base their hypotheses solely on the amount and content of the information given. Similar equalities have been shown for Gold-style and \emph{set-driven} learners ($\Sd$), which base their hypotheses solely on the content of the information given to them \citep{WC80}, by \citet{KS95} and \citet{KP16}. However, these results do not include witness-based learning. 

In this paper we expand these results to also include witness-based learners and generalize them such that they also hold for \emph{total} and \emph{globally} witness-based learners. This way, we discover interesting connections between these different ``types'' of learners and provide normal forms thereof. Furthermore, we also study \emph{behaviourally correct} learners ($\Bc$), which need to converge semantically to the correct language \citep{CL82,OW82}. Here, the mutual relation between the considered restrictions is yet to be discovered. Studying \emph{semantically witness-based} learners ($\SemWb$), the semantic counterpart of witness-based learners \citep{KSS17}, we complement similar studies of (target-) cautious learners conducted by \citet{DoskocK20} and, thus, get one step closer to discovering these relations.

In particular, in Section~\ref{Sec:Wb}, we extend the results of \citet{KS95} as well as \citet{KP16} regarding conservative, (target-) cautious as well as weakly monotone $\Ex$-learners to include witness-based learning. Simultaneously, we generalize these together with the results of \citet{KS16} to hold, amongst others, also for total and globally witness-based learners. These results are presented in Theorems~\ref{Thm:AllG},~\ref{Thm:AllPsd} and~\ref{Thm:AllSd} for $\G$-, $\Psd$- and $\Sd$-learners, respectively. Besides covering multiple types of learners at once, these results unveil interesting relations. For example, in the case of $\G$- or $\Psd$-learning, a total learner may be assumed globally witness-based, however, only maintaining its learning power for languages it learns target-cautiously. 

In Section~\ref{Sec:SemWb}, we study semantically witness-based $\Bc$-learners and show that three normal forms can be assumed \emph{simultaneously}. 
In particular, we show that \emph{semantically conservative} learners ($\SemWb$), the semantic counterpart of conservative learners \citep{KSS17} and a generalization of semantically witness-based learning, may be assumed (a) globally (b) semantically witness-based and (c) set-driven, see Theorem~\ref{Thm:tauSemWbSd-GSemConv}. In Section~\ref{Sec:Concl}, we conclude this work.

\section{Language Learning in the Limit}

\subsection{Preliminaries} \label{Sec:Prelim}

In this section we introduce notation and preliminary results used throughout this paper. Thereby, we consider basic computability theory as known, for an overview we refer the reader to \citet{Rogers87}. We start with the mathematical notation and use $\subsetneq$ and $\subseteq$ to denote the proper subset and subset relation between sets, respectively. We denote the set of all natural numbers as $\N = \{ 0, 1, 2, \dots \}$. Furthermore, we let $\partialCp$ and $\totalCp$ be the set of all partial and total computable functions $p\colon \N \to \N$. Next, we fix an effective numbering $\{\varphi_e\}_{e \in \N}$ of all partial computable functions and denote the $e$-th computably enumerable set as $W_e = \dom(\varphi_e)$ and interpret the number $e$ as an \emph{index} or \emph{hypothesis} of this set. Additionally, we mention the following important (total) computable functions. Firstly, we fix a total computable coding function $\langle . , . \rangle$ and its inverse for the first and second component $\pi_1$ and $\pi_2$, respectively. Furthermore, we write $\pad$ for an injective computable function such that, for all $e, k \in \N$, we have $W_e = W_{\pad(e, k)}$. We use $\unpad_1$ and $\unpad_2$ to recover the first and second component, respectively. Note that both functions can be extended iteratively to more coordinates. Lastly, for any finite set $D \subseteq \N$, we let $\ind(D)$ be an index for this finite set, that is, $W_{\ind(D)} = D$.

We learn recursively enumerable sets $L \subseteq \N$, called \emph{languages}, using \emph{learners}, that is, partial computable functions. By $\#$ we denote the \emph{pause symbol} and for any set $S$ we denote $S_\# \coloneqq S \cup \{ \# \}$. Then, a \emph{text} is a total function $T \colon \N \to \N \cup \{ \# \}$ and the collection of all texts is denoted as $\Txt$. In addition, for any text or sequence $T$, we let $\content(T) \coloneqq \range(T) \setminus \{ \# \}$ be the \emph{content} of $T$. A text of a language $L$ is such that $\content(T) = L$. We denote the collection of all texts of $L$ as $\Txt(L)$. Additionally, for $n \in \N$, we denote by $T[n]$ the initial sequence of $T$ of length $n$, that is, $T[0] \coloneqq \varepsilon$ and $T[n] \coloneqq (T(0), T(1), \dots, T(n-1))$. For a set $S$, we call the sequence (text) where all elements of $S$ are presented in strictly increasing order without interruptions (followed by infinitely many pause symbols if $S$ is finite) the \emph{canonical sequence (text) of $S$}. On finite sequences we use $\subseteq$ to denote the \emph{extension relation} and $\leq$ to denote the order on sequences interpreted as natural numbers. Given two sequences $\sigma$ and $\tau$ we write $\sigma\concat\tau$ to denote the concatenation of these. Occasionally, we omit writing $\concat$ to favour readability.

Following the system introduced by \citet{Kotzing09}, we formalize learning criteria. An \emph{interaction operator} $\beta$ takes a learner $h \in \partialCp$ and a text $T \in \Txt$ as argument and outputs a possibly partial function $p$. Intuitively, $\beta$ provides the information for the learner to make its guesses. We consider the interaction operators $\G$ for \emph{Gold-style} or \emph{full-information} learning \citep{Gold67}, $\Psd$ for \emph{partially set-driven} learning \citep{BlumBlum75,SchRicht84} and $\Sd$ for \emph{set-driven} learning \citep{WC80}. We define these using eponymous functions which operate on sequences called \emph{sequence interaction functions}. Define, for any $i \in \N$,
\begin{align*}
  \G(h, T)(i) &\coloneqq h(\G(T[i])), \text{where } \G(T[i]) = T[i], \\
  \Psd(h, T)(i) &\coloneqq h(\Psd(T[i])), \text{where } \Psd(T[i]) = (\content(T[i]), i), \\
  \Sd(h, T)(i) &\coloneqq h(\Sd(T[i])), \text{where } \Sd(T[i]) = \content(T[i]).
\end{align*}
The intuition is the following. A Gold-style learner has full information on the elements presented to it, while a partially set-driven learner does not have information on the order the elements were presented in or the frequency of each particular element. However, it may base its guess on the total amount of elements presented, an information a set-driven learner is not aware of. Furthermore, for each of the considered interaction operators we define an ordering $\preceq_\beta$ using the associated sequence interaction functions as follows. Given finite sequences $\sigma$ and $\tau$, we define
\begin{align}
  \sigma \preceq_\beta \tau :\Leftrightarrow \beta(\sigma) \preceq \beta(\tau) :\Leftrightarrow \exists \tau' \colon \beta(\sigma\concat\tau') = \beta(\tau). \label{preceq_beta}
\end{align}
Intuitively, $\sigma \preceq_\beta \tau$ indicates that information $\beta(\sigma)$ can be extended to $\beta(\tau)$. 

Given a learning task, we can distinguish between various criteria for successful learning. Initially, \citet{Gold67} introduced \emph{explanatory} learning ($\Ex$) as such a learning criterion, where the learner is expected to converge to a single, correct hypothesis in order to learn a language. This can be loosened to require the learner to converge semantically, that is, from some point onwards it must output correct hypotheses which may change syntactically. This is referred to as \emph{behaviourally correct} learning and denoted by $\Bc$ \citep{CL82,OW82}. Formally, a \emph{learning restriction} $\delta$ is a predicate on a total learning sequence $p$, that is, a total function, and a text $T \in \Txt$. For the mentioned criteria we have
\begin{align*}
  \Ex(p, T) &:\Leftrightarrow \exists n_0 \forall n \geq n_0 \colon p(n) = p(n_0) \wedge W_{p(n_0)} = \content(T), \\
  \Bc(p, T) &:\Leftrightarrow \exists n_0 \forall n \geq n_0 \colon W_{p(n)} = \content(T). 
\end{align*}
We can impose restrictions on the learners in order to model natural learning restrictions or such found in other sciences. For example, we consider \emph{consistent} learning ($\Cons$), where each hypothesis has to include the information it is built on, see \citet{Angluin80}. We focus on \emph{(semantically) witness-based} learners \citep{KP16,KSS17}, which need to \emph{justify} each of their (semantic) mind changes. These learners specialize a variety of important learning restrictions, such as \emph{conservative} learning ($\Conv$),  \emph{weakly monotone} learning ($\WMon$)  as well as \emph{cautious} learning ($\Caut$). While being consistent with the information given, conservative learners must not change their mind, see \citet{Angluin80}, and weakly monotone learners may not discard elements from their hypotheses, see \citet{Jantke91} and \citet{Wiehagen91}. Lastly, cautious learners may never fall back to a proper subset of any previous guess, see \citet{OSW82}. Generalizing these are \emph{target-cautious} learners ($\CautTar$), which may never overgeneralize the target language, see \citet{KP16}. Particular attention will be given to \emph{semantically conservative} learners ($\SemConv$), the semantic counterpart of conservative learners \citep{KSS17}.
We formalize the relevant restrictions as
\begin{align*}
  \Wb(p, T) &:\Leftrightarrow \forall n, m \colon (\exists k \colon n \leq k \leq m \wedge {p(n)} \neq {p(k)}) \Rightarrow \\
  &\phantom{\forall n, m \colon \exists k \}colon n} \Rightarrow (\content(T[m]) \cap W_{p(m)}) {\setminus} W_{p(n)} \neq \emptyset, \\
  \SemWb(p, T) &:\Leftrightarrow \forall n, m \colon (\exists k \colon n \leq k \leq m \wedge W_{p(n)} \neq W_{p(k)}) \Rightarrow \\
  &\phantom{\forall n, m \colon \exists k \}colon n} \Rightarrow (\content(T[m]) \cap W_{p(m)}) {\setminus} W_{p(n)} \neq \emptyset, \\
  \Cons(p,T) &:\Leftrightarrow \forall n \colon \content(T[n]) \subseteq W_{h(T[n])}, \\
  \CautTar(p, T) &:\Leftrightarrow \forall n \colon \neg(\content(T) \subsetneq W_{p(n)}), \\
  \SemConv(p, T) &:\Leftrightarrow \forall n, m \colon \left( n < m \wedge \content(T[m]) \subseteq W_{p(n)} \right) \Rightarrow W_{p(n)} = W_{p(m)}.
\end{align*}
Finally, the always true predicate $\T$ denotes the absence of a restriction.

Now, a \emph{learning criterion} is a tuple $(\alpha, \mathcal{C}, \beta, \delta)$, where $\mathcal{C}$ is a set of admissible learners, typically $\partialCp$ or $\totalCp$, $\beta$ is an interaction operator and $\alpha$ and $\delta$ are learning restrictions. We denote this learning criterion as $\tau(\alpha)\mathcal{C}\Txt\beta\delta$. In the case of $\mathcal{C} = \partialCp$, $\alpha= \T$ or $\delta=\T$ we omit writing the respective symbol. For an admissible learner $h \in \mathcal{C}$ we say that $h$ $\tau(\alpha)\mathcal{C}\Txt\beta\delta$-learns a language $L$ if and only if on arbitrary text $T \in \Txt$ we have $\alpha(\beta(h,T),T)$ and on texts of the target language $T \in \Txt(L)$ we have $\delta(\beta(h,T),T)$. With $\tau(\alpha)\mathcal{C}\Txt\beta\delta(h)$ we denote the class of languages $\tau(\alpha)\mathcal{C}\Txt\beta\delta$-learned by $h$ and the set of all such classes we denote with $[\tau(\alpha)\mathcal{C}\Txt\beta\delta]$.

\subsection{Normal Forms in Inductive Inference}

In this section we discuss helpful normal forms in inductive inference. Except for consistency, all introduced learning restrictions are \emph{delayable}. Informally, the hypotheses of a delayable restriction may be postponed arbitrarily but not indefinitely. Formally, we call a learning restriction $\delta$ \emph{delayable} if and only if for all texts $T$ and $T'$ with $\content(T) = \content(T')$, all learning sequences $p$ and all total, unbounded non-decreasing functions $r$, we have that if $\delta(p,T)$ and, for all $n$, $\content(T[r(n)]) \subseteq \content(T'[n])$, then $\delta( p \circ r, T')$. Furthermore, we call a restriction \emph{semantic} if and only if for any learning sequences $p$ and $p'$ and any text $T$, we have that if $\delta(p,T)$ and, for all $n$, $W_{p(n)} = W_{p'(n)}$ implies $\delta(p',T)$. Intuitively, a restriction is semantic if any hypothesis could be replaced by a semantically equivalent one without violating the learning restriction. Adding the requirement that no new syntactic mind change may be introduced by this replacement, we call a restriction \emph{pseudo-semantic} \citet{Kotzing17}. Note that all considered restrictions are pseudo-semantic and all, but $\Ex, \Wb$ and $\Conv$, are semantic. Delayable and semantic restrictions are of particular interest as one can provide general results for them. The following theorem holds.

\begin{theorem}[\normalfont{\cite{KP16}; \cite{KSS17}}]\label{Thm:DelTotal}
  For all interaction operators $\beta$, all delayable restrictions $\delta$ and all semantic restrictions $\delta'$, we have that
  \begin{align*}
    [\totalCp\Txt\G\delta] &= [\Txt\G\delta], \\
    [\totalCp\Txt\beta\delta'] &= [\Txt\beta\delta'].
  \end{align*}
\end{theorem}

Lastly, we discuss locking sequences. Intuitively, a locking sequence is a sequence where the learner correctly identifies the target language and does not make a mind change anymore regardless what information of the language it is presented. Formally, given a language $L$ and a $\G$-learner $h$, a sequence $\sigma \in L_\#^*$ is called a \emph{locking sequence} for $h$ on $L$ if and only if for every sequence $\tau \in L_\#^*$ we have that $h(\sigma) = h(\sigma\tau)$ and $W_{h(\sigma\tau)} = L$, see \citet{BlumBlum75}. Dropping the first requirement, $\sigma$ becomes a \emph{$\Bc$-locking sequence}, the semantic counterpart of a locking sequence, see \citet{JORS99}. When talking about $\Psd$-learners, for finite $D \subseteq \N$ and $t \in \N$, we call $(D,t)$ a \emph{locking information} if and only if for all $(D',t')$, such that $(D,t) \preceq (D',t')$ (compare Equation~\eqref{preceq_beta}) and $D' \subseteq L$, we have $h(D,t) = h(D',t')$ and $W_{h(D',t')} = L$. Lastly, for an $\Sd$-learner, a finite set $D$ is a \emph{locking set} of $L$ if and only if for all $D'$, with $D \subseteq D' \subseteq L$, we have $h(D) = h(D')$ and $W_{h(D')} = L$. Again, the semantic counterpart is obtained by dropping the first requirement. We use the term ($\Bc$-) \emph{locking information} to subsume all these concepts. 

It is an important observation by \citet{BlumBlum75} that every learner $h$ has a locking information on every language it learns. However, not every text may have an initial segment which is a locking information. Learners which do have a locking information on every text of a language they learn are called \emph{strongly ($\Bc$-) locking} \citep{KP16}. Formally, a learner is \emph{strongly ($\Bc$-) locking} on some language $L$ if on every text $T \in \Txt(L)$ there exists $n$ such that $T[n]$ is a ($\Bc$-) locking sequence for $h$ on $L$. If $h$ is strongly ($\Bc$-) locking on every language it learns, we call $h$ \emph{strongly ($\Bc$-) locking}. The transition to partially set-driven and set-driven learners is immediate and, thus, omitted.

\section{Normal Forms for Witness-Based Explanatory Learners} \label{Sec:Wb}

In this section we study witness-based learning in the explanatory setting. \citet{KS16} already show that partially set-driven witness-based learners are equally powerful as target-cautious $\Psd$-learners. Together with the result of \citet{DoskocK20}, where Gold-style target-cautious learners may be assumed partially set-driven, we see this holding true for full-information learners as well. However, general results including total or globally witness-based learners are still rare. With Theorems~\ref{Thm:AllG},~\ref{Thm:AllPsd} and~\ref{Thm:AllSd}, we provide expansions and generalizations of these results for Gold-style, partially set-driven and set-driven learners, respectively. 

We aim to generalize the results in two ways. Firstly, the aforementioned results were shown solely for partial learners, using additional total computable functions in order to get the desired equality. We observe that, for these constructions to work out, it suffices that the learner remains the ``same type'' after composition with total computable functions. This motivates the following notion of \emph{$\totalCp$-monoids}.
\begin{definition}
  We call $\Ia \subseteq \partialCp$ a \emph{\Rmonoid} if and only if $(\Ia, \circ)$ is a monoid and $\totalCp \subseteq \Ia$.
\end{definition}

Intuitively, the composition of a learner $h \in \Ia$ with other functions in $\Ia$ (especially total computable ones) shall remain in $\Ia$. This way, we can use functions obtained from, for example, the S-m-n Theorem, see \citet{Rogers87}, while keeping the learner's ``type''. Note that both $\totalCp$ and $\partialCp$ are $\totalCp$-monoids. 

Secondly, we generalize the domain on which we expect the learner to show a certain behaviour. The desire therefore arises from the observation that some of the learners behaviour does not rely on the learnable languages, but rather comes additionally. To require the learner to fulfil additional requirements on such information, we extend the notion of texts to classes of sets.

\begin{definition}
  Let $\Sa \subseteq \Pow(\N)$ be closed under subsets. We define $\Txt(\Sa) \subseteq \Txt$ as the set of all texts of elements of $\Sa$, that is, 
  \[
    \Txt(\Sa) \coloneqq \bigcup_{S \in \Sa} \Txt(S).
  \]
\end{definition}

If $\Sa = \emptyset$ then we consider no additional text, if $\Sa = \Pow(\N)$ then we consider all texts and, thus, the corresponding restriction becomes globally required. Already with the next result, we show the gain we have from these notions. Due to \citet{KP16}, it is known that Gold-style and set-driven learners may be assumed \emph{syntactically decisive}, that is, they never return to syntactically abandoned hypotheses. We are able to capture this result containing many different learners, including learners restricted in memory as well as total ones or learners obeying a further restriction on additional text, within a \emph{single} theorem.

\begin{theorem}\label{thm:ThmSynDec}
  Let $\delta$ and $\delta'$ be pseudo-semantic restrictions, $\beta \in \{\G, \Psd, \Sd \}$ and $\Sa \subseteq \Pow(\N)$ be closed under subsets. Let $\Ia$ be an {\Rmonoid} and let $h \in \Ia$. Let $\La$ be the class of languages $h$ $\Txt\beta\delta\Ex$-learns while being defined and $\delta'$ also on $\Txt(\Sa)$. Then, $h \in \Ia$ can be assumed syntactically decisive on $\Txt(\La)$ and $\Txt(\Sa)$.
\end{theorem}

\begin{proof}
  We generalize the ideas of \citet{KP16}, where a similar result has been shown for $\G$- and $\Sd$-learners. Let $h$ be as stated in the theorem. For a finite sequence $\sigma$, we define
  \begin{align*}
    Q(\beta(\sigma)) &\Leftrightarrow \forall \sigma', \beta(\sigma') \preceq \beta(\sigma)\colon \\ 
    & \phantom{\Leftrightarrow \forall \sigma',} h(\beta(\sigma')) = h(\beta(\sigma)) \Rightarrow \forall \sigma'', \beta(\sigma'') \in [\beta(\sigma'), \beta(\sigma)]\colon h(\beta(\sigma'')) = h(\beta(\sigma)), \\
    \tau_\sigma &= \min \{ \sigma' \mid \beta(\sigma') \preceq \beta(\sigma) \wedge \forall \sigma'', \beta(\sigma'') \in [\beta(\sigma'), \beta(\sigma)]\colon h(\beta(\sigma'')) = h(\beta(\sigma)) \}, \\
    h'(\beta(\sigma)) &= \begin{cases} \pad(h(\beta(\sigma)), 0), &\falls Q(\beta(\sigma)), \\ \pad(h(\beta(\tau_\sigma)), \beta(\tau_\sigma)), &\sonst \end{cases}
  \end{align*}
  Without loss of generality, we may assume that we do not encode $\beta(\tau_\sigma)$ as $0$. Intuitively, $Q(\beta(\sigma))$ checks whether $h(\beta(\sigma))$ has been conjectured without interruptions. In this case, $h$ conjectures $\pad(h(\beta(\sigma)), 0)$, where the second component zero indicates the absence of interruptions. Otherwise, the second component becomes $\beta(\tau_\sigma)$, which is the minimal information on which no mind change has been witnessed. Note that the first component is, by definition of $\tau_\sigma$, the same in both cases. We first show syntactic decisiveness. Let therefore $\sigma_1, \sigma_2, \sigma_3$ be sequences such that $\beta(\sigma_1) \preceq \beta(\sigma_2) \preceq \beta(\sigma_3)$ and
  \begin{align}
    h'(\beta(\sigma_1)) = h'(\beta(\sigma_3)). \label{Eq:SynEqv}
  \end{align}
  To show $h'(\beta(\sigma_1)) = h'(\beta(\sigma_2))$, we distinguish between the following cases.
  \begin{enumerate}
    \itemin{1. Case:} $\unpad_2(h'(\beta(\sigma_3))) = 0$. Then, we have $\pad(h(\beta(\sigma_1)), 0) = h'(\beta(\sigma_1)) = h'(\beta(\sigma_3)) = \pad(h(\beta(\sigma_3)), 0)$, meaning that we have, by definition of $h'$, $Q(\beta(\sigma_1))$ and $Q(\beta(\sigma_3))$ as well as $h(\beta(\sigma_1)) = h(\beta(\sigma_3))$ since the padding function is injective. As $Q(\beta(\sigma_3))$ holds, we have that, for all $\sigma'$ such that $\beta(\sigma') \preceq \beta(\sigma_3)$,
    \[
      h(\beta(\sigma')) = h(\beta(\sigma_3)) \Rightarrow \forall \sigma'', \beta(\sigma'') \in [\beta(\sigma'), \beta(\sigma_3)] \colon h(\beta(\sigma'')) = h(\beta(\sigma_3)).
    \]
    As $h(\beta(\sigma_1)) = h(\beta(\sigma_3))$, this, in particular, holds true for $\sigma' = \sigma_1$. Choosing $\sigma'' = \sigma_2$, we get $h(\beta(\sigma_2)) = h(\beta(\sigma_3))$. Since $Q(\beta(\sigma_3))$ and $\beta(\sigma_2) \preceq \beta(\sigma_3)$, it also holds that $Q(\beta(\sigma_2))$. Thus, we get $h'(\beta(\sigma_2)) = \pad(h(\beta(\sigma_2)), 0)$ which is equal to $h'(\beta(\sigma_3))$, as desired.
    \itemin{2. Case:} $\unpad_2(h'(\beta(\sigma_3))) \neq 0$. In this case, unpadding the second components of the hypotheses in Equation~\eqref{Eq:SynEqv}, we get $\beta(\tau_{\sigma_1}) = \beta(\tau_{\sigma_3})$, meaning that for $\tau_{\sigma_3}$ we have $\beta(\tau_{\sigma_3}) \preceq \beta(\sigma_1)$ and
    \begin{align}
      \forall \sigma'', \beta(\sigma'') \in [\beta(\tau_{\sigma_3}), \beta(\sigma_3)] \colon h(\beta(\sigma'')) = h(\beta(\sigma_3)). \label{Eq:SynDecCond}
    \end{align}
    Choosing $\sigma'' = \sigma_2$, in particular, we get $h(\beta(\sigma_1)) = h(\beta(\sigma_2)) = h(\beta(\sigma_3))$. From Equation~\eqref{Eq:SynDecCond}, one can also easily see that $\beta(\tau_{\sigma_2}) = \beta(\tau_{\sigma_3})$. Now, as $\neg Q(\beta(\sigma_1))$, there exist $\sigma, \tilde{\sigma}$ such that $\beta(\sigma) \preceq \beta(\tilde{\sigma}) \preceq \beta(\sigma_1)$ and $h(\beta(\sigma)) = h(\beta(\sigma_1)) \neq h(\beta(\tilde{\sigma}))$. In particular, $\sigma, \tilde{\sigma}$ are such that $\beta(\sigma) \preceq \beta(\tilde{\sigma}) \preceq \beta(\sigma_2)$ and $h(\beta(\sigma)) = h(\beta(\sigma_2)) \neq h(\beta(\tilde{\sigma}))$, meaning that $\neg Q(\beta(\sigma_2))$. Thus, we get $h'(\beta(\sigma_2)) = h'(\beta(\sigma_3))$ as desired.
  \end{enumerate}
  We next show $\Ex$-convergence. Let $L \in \La$ and $T \in \Txt(L)$. Let $n_0$ be minimal such that $W_{h(\beta(T[n_0]))} = L$ and, for all $n \geq n_0$, we have $h(\beta(T[n_0])) = h(\beta(T[n]))$. Furthermore, let $n_1 \geq n_0$ be minimal such that, for all $n \geq n_1$, we have $\tau_{T[n_1]} = \tau_{T[n]}$. Such $n_1$ exists as $\N$ with $\leq$ on sequences is a well-order. Now, for $n \geq n_1$, $h'(\beta(T[n]))$ either conjectures $\pad(h(\beta(T[n])), 0)$ if no interruption has been witnessed. Otherwise, there exists $n_2 \geq n_1$ where an interruption has been witnessed. Then, for all $n \geq n_2$, we have that $h'(\beta(T[n]))$ conjectures $\pad(h(\beta(\tau_{T[n]})), \beta(\tau_{T[n]}))$. Note that this does not imply infinitely many mind changes, as we have that $n \geq n_1$ and, with it, $\tau_{T[n_1]} = \tau_{T[n]}$. In both cases, we have that $h'$ $\Ex$-learns $L$ from text $T$.

  Since $h'$ only makes mind changes when it witnesses a mind change or interruption of $h$ and since $h$ serves the restrictions $\delta$ on $\Txt(\La)$ and $\delta'$ on $\Txt(\Sa)$, we also have that $h'$ serves those. This concludes the proof.
\end{proof}

We make use of this result when proving our main generalizations. However, we first have to deal with the following issue. Ideally, when generalizing the result of \citet{KP16}, we would obtain that any (possibly globally) target-cautious $\G$-learner is as powerful as a (possibly globally) witness-based one. However, the issue arises as globally target-cautious learners are extremely weak. In fact, as they may never overgeneralize they can solely learn finite languages, as we show in the following theorem.

\begin{theorem}
  Let $\La \subseteq \Pow(\N)$. We have that
  \[
    \La \subseteq \Pow_\Fin(\N) \Leftrightarrow \La \in [\tau(\CautTar)\Txt\Sd\Ex] \Leftrightarrow \La \in [\tau(\CautTar)\Txt\G\Bc].
  \]
\end{theorem}

\begin{proof}
  As the learner $h$ which, for all finite $D \subseteq \N$, is defined as $h(D) = \ind(D)$ learns $\Pow_\Fin(\N)$, we have that 
  $$\La \subseteq \Pow_\Fin(\N) \Rightarrow \La \in [\tau(\CautTar)\Txt\Sd\Ex].$$ 
  As $[\tau(\CautTar)\Txt\Sd\Ex] \subseteq [\tau(\CautTar)\Txt\G\Bc]$, we also have 
  $$\La \in [\tau(\CautTar)\Txt\Sd\Ex] \Rightarrow \La \in [\tau(\CautTar)\Txt\G\Bc].$$
  It remains to be shown that if $\La \in [\tau(\CautTar)\Txt\G\Bc]$ then $\La \subseteq \Pow_\Fin(\N)$. We show this by contradiction. To that end, let $h$ be a learner such that $\La \subseteq \tau(\CautTar)\Txt\G\Bc(h)$ and assume there exists an infinite $L \in \La$. Let $T \in \Txt(L)$, and let $n$ be such that $W_{h(T[n])} = L$. Now, considering the text $T' \coloneqq T[n] ^\frown \#^\infty$, we have
  \[
    \content(T') = \content(T[n]) \subsetneq L = W_{h(T[n])} = W_{h(T'[n])}.
  \]
  Thus, $h$ is not $\CautTar$ on text $T'$. This is a contradiction, so we have
  $$\La \in [\tau(\CautTar)\Txt\G\Bc] \Rightarrow \La \subseteq \Pow_\Fin(\N).$$
  This concludes the proof.
\end{proof}

We overcome this issue by observing that witness-based behaviour may be obtained whenever the learner is defined, solely on languages it is supposed to learn it needs to be target-cautious. The idea is that, when being target-cautious, any wrong guess may still be changed when observing a missing element from the target-language. Now we can provide the general versions of the discussed results of \citet{KS95}, \citet{KP16} and \citet{KS16}. Starting with Gold-style learners, \citet{KP16} show that target-cautious Gold-style learners may be assumed weakly monotone, cautious and conservative without losing learning power. We show that these even may be assumed witness-based. Furthermore, within a \emph{single} theorem, we show that this is not only the case when dealing with partial learners.

\begin{theorem}\label{Thm:AllG}
  Let $\Ia$ be an {\Rmonoid} and $\Sa \subseteq \Pow(\N)$ closed under subsets. Let $\La$ be a class of languages. Then, the following are equivalent.
  \begin{enumerate}[label=\textnormal{(\arabic*)}]
    \item $\La$ can be $\Txt\G\Wb\Ex$-learned by an $\Ia$-learner which is $\Wb$ also on $\Txt(\Sa)$. \label{GWb}
    \item $\La$ can be $\Txt\G\CautTar\Ex$-learned by an $\Ia$-learner which is defined also on $\Txt(\Sa)$. \label{GCT}
  \end{enumerate}
\end{theorem}

\begin{proof}
  The direction \ref{GWb}$\Rightarrow$\ref{GCT} follows immediately. For the other direction, we follow and expand the proof of $[\Txt\G\Conv\Ex] = [\Txt\G\CautTar\Ex]$, see \citet[Thm.~13]{KP16}. Let $h \in \Ia$ be a learner and $\La$ be $\Txt\G\CautTar\Ex$-learnable by $h$ which is also defined on $\Txt(\Sa)$. By Theorem~\ref{thm:ThmSynDec} we can assume $h \in \Ia$ to be syntactically decisive on $\Txt(\La)$ and $\Txt(\Sa)$. Let $p \in \totalCp$ be such that
  \begin{align*}
    W_{p(\sigma)} = \bigcup_{t \in \N} \begin{cases} 
      \content(\sigma), &\falls \neg (\content(\sigma) \subseteq W_{h(\sigma)}^t), \\ 
      W_{h(\sigma)}^t, &\sonstfalls \forall \rho \in (W_{h(\sigma)}^t)^{\leq t}\colon h(\sigma) = h(\sigma ^\frown \rho), \\ 
      \emptyset, &\sonst 
    \end{cases}
  \end{align*}
  Note that for all sequences $\sigma$ we have
  \begin{align}
    W_{p(\sigma)} \subseteq \content(\sigma) \cup W_{h(\sigma)}. \label{pSubseth}
  \end{align}
  For given sequences $\sigma$ and $\tau$, we define 
  \[
    \tau \trianglelefteq \sigma :\Leftrightarrow \content(\tau) \subseteq \content(\sigma) \wedge |\tau| \leq |\sigma|. 
  \] 
  Given finite sequences $\sigma, \sigma'$ and $\tau$, we define the computable predicate 
  \[
    Q_\sigma(\sigma', \tau) \Leftrightarrow \content(\sigma') \subsetneq \content(\sigma) \wedge h(\sigma') \neq h(\tau) \wedge \content(\tau) \not\subseteq W_{h(\sigma')}^{|\tau|-1}.
  \]
  Given a sequence $\sigma$ we can define the learner $h'$ as follows. For convenience, for $\sigma \neq \varepsilon$, let $\sigma'$ be such that $h'(\sigma^-) = p(\sigma')$. Then, we define
  \begin{align*}
    h'(\sigma) = \begin{cases} 
      p(\varepsilon), &\falls \sigma = \varepsilon, \\ 
      p(\tau ^\frown \sigma), &\sonstfalls \exists \tau, \sigma' \subseteq \tau \trianglelefteq \sigma\colon Q_\sigma(\sigma', \tau), \\
      h'(\sigma^-), &\sonst 
    \end{cases}
  \end{align*}
  We motivate the intuition behind the learner $h'$. Given a sequence $\sigma \neq \varepsilon$, let $\sigma'$ be the sequence $h'$ based its previous guess on. Then, if not consistent, $h'$ only changes its mind if there exists an extension $\tau$ of $\sigma'$ on which $h$ made a mind change and if there exist elements in $\content(\tau)$ not yet enumerated by $W_{h(\sigma')}^{|\tau|-1}$.

  We first show that $h'$ converges on any text for a language $L \in \La$. Let $L \in \La$ and $T \in \Txt(L)$. Assume $h'$ does not converge on $T$. Let $(p(\sigma_i'))_{i \in \N}$ be the sequence of hypotheses output by $h'$ on text $T$. As $h'$ makes infinitely many mind changes, the sequence $(p(\sigma_i'))_{i \in \N}$ contains infinitely many different hypotheses. For any $i$ where $h'(T[i])$ makes a mind change, the corresponding hypothesis is, for apt $\tau_i$, $p({\tau_i}\concat T[i])$ instead of its previous hypothesis. Thus, in particular, for infinitely many $i$, $\sigma_{i+1}'$ has the form ${\tau_i}^\frown T[i]$. Thus, $T' = \bigcup_{i \in \N} \sigma_i'$ is a text of $L$. As $h'$ makes infinitely many mind changes, for each $\sigma_i'$ exists $\tau_i$, $\sigma_i' \subseteq \tau_i \subseteq \sigma_{i+1}'$ with $h(\sigma_i') \neq h(\tau_i)$, as seen in the second case of the definition of $h'$. This contradicts the convergence of $h$ on $T'$. 

  Next, we show that $h'$ converges to the correct hypothesis on $T$. Let $n_0$ be the point of convergence, that is, $n_0$ is minimal such that for all $n \geq n_0$ we have $h'(T[n_0]) = h'(T[n])$. We abbreviate $\sigma \coloneqq T[n_0]$. Let furthermore $\sigma'$ be such that $h'$ converges to $p(\sigma')$, that is, $h'(\sigma) = p(\sigma')$. We consider the following two cases.
  \begin{enumerate}
    \itemin{1. Case:} $\sigma'$ is a locking sequence for $h$ on $L$. Then, we have $W_{h(\sigma')} = L$ and, for all $\rho \in (W_{h(\sigma')})^*$, we have $h({\sigma'} ^\frown \rho) = h(\sigma')$. Thus, we get $W_{p(\sigma')} = W_{h(\sigma')} = L$.
    \itemin{2. Case:} $\sigma'$ is not a locking sequence for $h$ on $L$. As $h'$ converges, we have for all $n$ and $\tau$ with $\sigma' \subseteq \tau \trianglelefteq T[n]$ that $\neg Q_\sigma(\sigma', \tau)$, that is, for all $n \in \N_{\geq n_0}$ and $\tau \in L^*$, with $\sigma' \subseteq \tau$, we have
    \[
      \content(\sigma') = \content(T[n]) \vee h(\sigma') = h(\tau) \vee \content(\tau) \not\subseteq W_{h(\sigma')}^{|\tau|-1}.
    \]
    We distinguish the following cases.
    \begin{enumerate}
      \itemin{2.1. Case:} For all $n \in \N_{\geq n_0}$ we have $\content(\sigma') = \content(T[n])$. Then, $\content(\sigma') = L$. Thus, $L \subseteq W_{p(\sigma')}$ by definition of $p$. Since $h$ is $\CautTar$, we have $\neg (L \subsetneq W_{h(\sigma')})$. Together with Inclusion~\eqref{pSubseth}, we get $L = W_{p(\sigma')}$.
      \itemin{2.2. Case:} For all $\tau \in L^*, \sigma' \subseteq \tau$ we have $h(\sigma') = h(\tau)$. So, $h$ never changes its mind on an extension of $\sigma'$ within the language $L$, but $\sigma'$ is no locking sequence either. This means that $W_{h(\sigma)} \neq L$, meaning that $h$ does not learn $L$ on any text starting with $\sigma'$, a contradiction. 
      \itemin{2.3. Case:} There exists $n \in \N_{\geq n_0}$ with $\neg (\content(\sigma') = \content(T[n]))$ and there exists $\tau \in L^*, \sigma' \subseteq \tau$ such that $h(\sigma') \neq h(\tau)$. Let $n_1$ be minimal such that $\neg (\content(\sigma') = \content(T[n_1]))$. Let $n \geq n_1$. Note that $\content(\sigma') \subsetneq \content(T[n])$ holds.

      We proceed by showing $L \subseteq W_{h(\sigma')}$ and afterwards $W_{p(\sigma')} = W_{h(\sigma')}$. As $h$ is $\CautTar$, this suffices in order to show $L = W_{p(\sigma')}$. Let $\tau \in L^*, \sigma' \subseteq \tau$ such that $h(\sigma') \neq h(\tau)$ as assumed to exist in this case. Then, for $x \in L \setminus \content(\sigma')$, we have $h(\sigma') \neq h(\tau \concat x)$ by syntactic decisiveness of $h$, meaning that, as $\neg Q_{T[n]}(\sigma', \tau\concat x)$, we have $\content(\tau \concat x) \subseteq W_{h(\sigma')}^{|\tau\concat x| - 1}$. Since this holds for all $x \in L \setminus \content(\sigma')$ and since $\content(\sigma') \subseteq \content(\tau\concat x)$, we get $L \subseteq W_{h(\sigma')}$.

      It remains to be shown $W_{p(\sigma')} = W_{h(\sigma')}$. The inclusion $W_{p(\sigma')} \subseteq W_{h(\sigma')}$ follows from Inclusion~\eqref{pSubseth} and, as just shown, $L \subseteq W_{h(\sigma')}$. For the other direction, assume there exists $x \in W_{h(\sigma')} \setminus W_{p(\sigma')}$. Then, there exists a minimal $t_x$ such that $x \in W_{h(\sigma')}^{t_x}$. But, as $x$ is not in $W_{p(\sigma')}$, there also is $\rho \in (W_{h(\sigma)}^{t_x})^*$, with $|{\sigma'}\concat \rho| \leq t_x$, such that $h(\sigma') \neq h({\sigma'}\concat\rho)$. We abbreviate $\tau \coloneqq{\sigma'}\concat \rho$.  Due to $h$ being syntactically decisive, we also have $h(\sigma') \neq h(\tau\concat x)$. By assumption, we have $\neg Q_{T[n]}(\sigma', \tau\concat x)$, that is, 
      \[
        \content(\sigma') = \content(T[n]) \vee h(\sigma') = h(\tau\concat x) \vee \content(\tau\concat x) \subseteq W_{h(\sigma')}^{|\tau\concat x| - 1}.
      \]
      However, as $n>n_1$, $\content(\sigma') = \content(T[n])$ does not hold and neither does $h(\sigma') = h(\tau\concat x)$. Thus, $\content(\tau\concat x) \subseteq W_{h(\sigma')}^{|\tau\concat x| - 1}$. But the fact that $x \in W_{h(\sigma')}^{|\tau\concat x| - 1}$ and $|\tau\concat x| - 1 < t_x$ contradict the choice of $t_x$.
    \end{enumerate}
  \end{enumerate}
  Finally, we show that $h'$ is witness-based on $\Txt(\La)$ and $\Txt(\Sa)$. Let $\sigma$ be a sequence where $h'$ makes a mind change, that is, for some $\tau$ with $\sigma' \subseteq \tau \trianglelefteq \sigma$ such that $Q_\sigma(\sigma', \tau)$, we have $p(\sigma') = h'(\sigma^-) \neq h'(\sigma) = p(\tau^\frown \sigma)$. We show that 
  \[
    \left(\content(\sigma) \cap W_{h'(\sigma)}\right) \setminus W_{h'(\sigma^-)} \neq \emptyset.
  \]
  As, by definition of $p$, $W_{h'(\sigma)} = W_{p(\tau\concat\sigma)} \supseteq \content(\tau\concat\sigma) \supseteq \content(\sigma)$, it suffices to show 
  \[
    \content(\sigma) \setminus W_{h'(\sigma^-)} \neq \emptyset.
  \]
  Furthermore, as later hypotheses of $h'$ are built on extensions of $\tau\concat\sigma$, this is sufficient in order to show that $h'$ is witness-based. We proceed with showing $\content(\sigma) \setminus W_{h'(\sigma^-)} \neq \emptyset$. Assume the opposite, that is, $\content(\sigma) \subseteq W_{h'(\sigma^-)} = W_{p(\sigma')}$. Since $Q_\sigma(\sigma', \tau)$, we have that $\content(\sigma') \subsetneq \content(\sigma)$, $h(\sigma') \neq h(\tau)$ and $\content(\tau) \not\subseteq W_{h(\sigma')}^{|\tau|-1}$. In particular, as $\content(\sigma') \subsetneq \content(\sigma)$, $W_{p(\sigma')}$ has to enumerate $\content(\sigma)$ by means of the second condition in the definition of $p$. Let $t > |\tau| - 1$ be a step where such a enumeration could take place, that is, $W_{h(\sigma')}^t \supseteq \content(\sigma)$. The condition $t > |\tau| - 1$ follows from the third condition of $Q_\sigma(\sigma', \tau)$, that is, $\content(\tau) \not\subseteq W_{h(\sigma')}^{|\tau|-1}$. Now, there exists some $\rho \in (W_{h(\sigma')}^t)^{\leq t}$ such that $\tau = {\sigma'} \concat \rho$. However, $h(\sigma') \neq h(\tau) = h({\sigma'}\concat\rho)$, meaning that $W_{p(\sigma')}$ does not enumerate $W_{h(\sigma')}^t$. Thus, $W_{p(\sigma')} \not\supseteq \content(\sigma)$, a contradiction.
\end{proof}

For example, if one is to take $\Ia = \totalCp$ (which is an {\Rmonoid}), one gains the same observation as for partial learners, namely that target-cautious learners may be assumed witness-based. However, this theorem provides another interesting result. With $\Ia = \totalCp$ and $\Sa = \Pow(\N)$, Theorem~\ref{Thm:AllG} shows that any anywhere defined learner (that is, any total learner) may be assumed \emph{everywhere} (that is, globally) witness-based. However, it maintains its learning power solely on languages it learns target-cautiously. With the observation of \citet{KP16} that any $\G$-learner may be assumed total, we get this property even for arbitrary learners instead of total ones. 

Using our general framework we extend this result even further. \citet{KS16} show that (possibly partial) target-cautious partially set-driven learners may be assumed witness-based. We observe that this also holds true for a variety of different learners. 

\begin{theorem}\label{Thm:AllPsd}
  Let $\Ia$ be an {\Rmonoid} and $\Sa \subseteq \Pow(\N)$ closed under subsets. Furthermore, let $\La$ be a class of languages. Then, the following are equivalent.
  \begin{enumerate}[label=\textnormal{(\arabic*)}]
    \item $\La$ can be $\Txt\Psd\Wb\Ex$-learned by an $\Ia$-learner which is $\Wb$ also on $\Txt(\Sa)$.\label{PsdWb}
    \item $\La$ can be $\Txt\Psd\CautTar\Ex$-learned by an $\Ia$-learner which is defined also on $\Txt(\Sa)$.\label{PsdCT}
  \end{enumerate}
\end{theorem}

\begin{proof}
  The direction \ref{PsdWb}$\Rightarrow$\ref{PsdCT} follows immediately. For the other direction, we follow the proof of $[\Txt\Psd\Wb\Ex] = [\Txt\Psd\CautTar\Ex]$, see \citet[Thm.~3.5]{KS16}. Let $h \in \Ia$ be a learner and $\La$ be $\Txt\Psd\CautTar\Ex$-learnable by $h$ which is also defined on $\Txt(\Sa)$. First, we show that we may assume $h$ to be strongly locking. 

  \begin{claim}
    The learner $h$ may be assumed strongly locking.
  \end{claim}
  \begin{proof}[Claim]
    We adapt the proof of \citet{KS16} for partial $\Psd$-learners. For any finite set $D$ and any number $t \geq |D|$, define the learner $\hat{h}(D,t) = h(D,2t)$. As $h \in \Ia$, so is $\hat{h}$. Furthermore, $\hat{h}$ is target-cautious as $h$ is. Let $L \in \La$ and let $T \in \Txt(L)$. Note that $h$ learns $L$ on text $\hat{T}$ defined as, for any $i \in \N$,
    \[
      \hat{T}(2i) = T(i) \wedge \hat{T}(2i + 1) = \#.
    \]
    Furthermore, note that, for any $n \in \N$, 
    \begin{align}
      \hat{h}^*(T[n]) = \hat{h}(\content(T[n]), n) = h(\content(\hat{T}[2n]), 2n) = h^*(\hat{T}[2n]). \label{Claim:StrLock}
    \end{align}
    As $h$ learns $L$ on text $\hat{T}$, there exists $n_0$ such that, for all $n \geq n_0$, we have $h^*(\hat{T}[n]) = h^*(\hat{T}[n_0])$ and $W_{h^*(\hat{T}[n])} = L$. In particular, this holds for all even $n$ and together with Equation~\eqref{Claim:StrLock}, we get that also $\hat{h}$ converges correctly on text $T$.

    It remains to be shown that $\hat{h}$ is strongly locking. Let $L \in \La$ and let $\sigma_0$ be a locking sequence of the starred learner $h^*$. Let $T \in \Txt(L)$. We show that $\hat{h}$ is locking on $T$. Let $n_0 \geq |\sigma_0|$ such that $\content(T[n_0]) \supseteq \content(\sigma_0)$. We show that $(\content(T[n_0]), n_0)$ is a locking information. Let $n \geq n_0$. As
    \begin{align*}
      |\content(T[n]) \setminus \content(\sigma_0)| + |\sigma_0| \leq |\content(T[n])| + |\sigma_0| \leq n + n_0 \leq 2n,
    \end{align*}
    there exists a sequence $\tau \in L^*$ such that $\content(\sigma \concat \tau) = \content(T[n])$ and $|\sigma_0| + |\tau| = 2n$. As $\sigma_0$ is a locking sequence for $h^*$ we get
    \[
      \hat{h}(\content(T[n]), n) = h(\content(T[n]), 2n) = h^*(\sigma \concat \tau) = h^*(\sigma).
    \]
    So, $\hat{h}$ is strongly locking. \claimqedhere \noqed
  \end{proof}

  Now, we build the desired $\Txt\Psd\Wb\Ex$-learner $h'$ which is also witness-based on $\Txt(\Sa)$ in the following way. Following \citet{KS16}, we mimic the strongly locking learner $h$. Given some information, the main idea is to \emph{delay until refutation}, that is, we wait with abandoning a hypothesis until we see some datum witnessing the current hypothesis being incorrect. Secondly, we \emph{search for locking information}, a method often applied when mimicking learners. This way, we assure correct convergence. Furthermore, to overcome the problem of multiple possible previous hypotheses, we assume the information appeared in strictly ascending order without pause symbols, that is, we simulate $h$ on \emph{ascending text}. Next, we \emph{poison} hypotheses when witnessing them being incorrect, that is, once a possible later mind change is witnessed, we stop enumerating more data into this hypothesis. Lastly, we \emph{delay until consistency}, that is, we do not accept a new hypothesis until we see it to be consistent with new data.

  For this sketched strategy to work, we have to overcome two major problems. Firstly, when simulating $h$ on ascending text, we run into problems learning finite languages, as the learner may learn such languages way after all data has been presented to it. We overcome this problem by searching for such mind changes and returning an index for the finite language in question, unless we have found a consistent hypothesis before. As $h$ is target-cautious, once a hypothesis overgeneralizes the information it got, the given information cannot be the final target.

  The second problem is that partially set-driven learners have multiple previous hypotheses. We solve this problem by assuming the information to be presented in strictly ascending order. However, this poses the following difficulty. For example, given the sequence $0,1,3,2$, our new learner may base its hypothesis on $\{0,1,2\}$ without ever considering $\{0,1,3\}$ as possible input. We oppose this problem in the following way. For finite $D$, we only use a conjecture $h(D, |D|)$ when it contains at least $D$ and one additional element. Furthermore, we remove all data from this this conjecture which is less than $\max(D)$, but not in $D$. This way, we may change the conjecture built on $\{0,1,3\}$ as it does not contain the $2$. Additionally, conjectures based on $\{0,1,2\}$ will contain all of $\{0,1,2\}$ as we only use conjectures which are consistent. Since we can determine consistency only in the limit, we have to delay conjectures until consistency.

  In order to define the desired learner, we continue with the formal details. For any given $D, D' \subseteq \N$ and $t' \in \N$ such that $(D, |D|) \preceq (D',t')$, define the predicate 
  \begin{align*}
    \refuted(D,D',t') \Leftrightarrow \exists (D'', t'') \colon (D, |D|) \preceq (D'', t'') \preceq (D', t') \wedge h(D, |D|) \neq h(D'', t'').
  \end{align*}
  The function $\refuted$ tells us whether there is a mind change between $(D, |D|)$ and $(D', t')$. Fix, for any set $S$ and any $x \in \N$, the notation $S_{> x}$ for the set of all elements in $S$ which are larger than $x$. The set $S_{\le x}$ is defined analogously. Furthermore, let $c \in \totalCp$ such that, for any hypothesis $e$ and any finite set $D$,
  \begin{align*}
    W_{c(e,D)} = (W_e)_{> \max(D)} \cup D.
  \end{align*}
  The function $c(e,D)$ enumerates the set $D$ itself and all elements of $W_e$ which are larger than $\max(D)$. Lastly, let $p \in \totalCp$ such that 
  \begin{align*}
    W_{p(D)} = \bigcup_{t \geq 0} \begin{cases} W_{h(D, |D|)}^t, &\falls \neg \refuted(D, D \cup W_{h(D, |D|)}^t, |D| + t +1), \\ \emptyset, &\sonst \end{cases}
  \end{align*}
  The function $p$ enumerates the same as $W_{h(D, |D|)}$ until, if ever, a mind change is witnessed. Then, it stops enumerating more data. We combine the previously introduced functions by letting, for all finite $D \subseteq \N$, $q(D) \coloneqq c(p(D),D)$. The function $q$ does exactly what we explained before: while no mind change is witnessed, it enumerates all of $W_{h(D, |D|)}$ which is larger than any element given and the set $D$ itself as well. 

  To assure convergence later on, we show for all $L \in \La$ and all $k \in \N$, we have that if $(L[k],k)$ is a locking information of $h$ on $L$, then $W_{q(L[k])} = L$. Let $(L[k],k)$ be a locking information. Thus, no refutation will be witnessed on any extending information of the language. Hence, $W_{p(L)} = W_{h(L[k], k)} = L$. Since $L[k]$ contains the information in ascending order, we have $(W_{h(L[k],k)})_{\leq \max(L[k])} = L[k]$. Thus, 
  \[
    W_{q(L[k])} = W_{c(p(L[k]), L[k])} = (W_{h(L[k],k)})_{> \max(L[k])} \cup L[k] = W_{h(L[k],k)} = L.
  \]

  In order to define that learner $h'$ we need the following decision procedure. We show that for $D, D'$ and $t'$ such that $D'[|D|] = D$ and $\refuted(D, D', t')$, there exists an algorithm taking $D, D'$ and $t'$ as input and deciding $D' \subseteq W_{q(D)}$. Namely, the algorithm computes the finite set
  \begin{align*}
    A = \bigcup_{t \leq t' - (|D|+1)} \begin{cases} W_{h(D, |D|)}^t, &\falls \neg \refuted(D, D \cup W_{h(D, |D|)}^t, |D| + t + 1), \\ \emptyset, &\sonst \end{cases}
  \end{align*}
  Now, either $D' \subseteq A$, in which case $D' \subseteq W_{q(D)}$. Otherwise $D' \not \subseteq W_{q(D)}$, as the enumeration would have to stop before enumerating the last element of $D'$ because it would make the mind change visible.

  Lastly, we need the following notation. For any set $D$ and any $k \in \N$, we use $D[k]$ to denote the set of the first $k$ elements of $D$ (in ascending order). If $k > |D|$, we let $D[k] = D$. Furthermore, let $\sigma_D$ be the sequence of elements in $D$ in strictly ascending order without pause-symbols. 

  \begin{algorithm2e}[t]
    \caption{Witness-based learner $h'$.}\label{Algo:PsdWbEx}
    \param{Learner $h$.}
    \input{Finite set $D \subseteq \N$ and $t \geq |D|$.}
    \outoutput{New hypothesis $h'(D,t)$.}
    \For{$k=0$ \KwTo $|D| - 1$ \label{AlgPsdWb_Start_1}}{
      \If{$\refuted(D[k], D, t) \wedge D \subseteq W_{q(D[k])}$ \label{AlgPsdWb_Start_2}}{
        \Return $q(D[k])$ \label{AlgPsdWb_Start_3}
      }
    }
    \If{$\forall k \leq |D| \colon \refuted(D[k], D, t)$ \label{AlgPsdWb_Ind_1}}{
      \Return $\ind(D)$ \label{AlgPsdWb_Ind_2}
    }
    $k_0 \leftarrow \min \{ k \leq |D| : \neg \refuted(D[k], D, t) \}$ \label{AlgPsdWb_knull} \\
    \For{$s = k_0$ \KwTo $t$}{
      \For{$k = k_0$ \KwTo $\min(s,|D|)$}{
        \If{$\forall k' < k_0\colon \refuted(D[k'], D[k], s) \wedge D[k] \not \subseteq W_{q(D[k'])}$}{
          \uIf{$k < |D| \wedge D[k+1] \subseteq W_{q(D[k_0])}^t$ \label{AlgoPsdWb_qCons}}{
            \Return $q(D[k_0])$ \label{AlgPsdWb_final_1}
          }\Else{
            \Return $\ind(D[k])$ \label{AlgPsdWb_final_2}
          }
        }
      }
    }
  \end{algorithm2e}

  Finally, we can formalise $h'$ as in Algorithm~\ref{Algo:PsdWbEx}. First, note that $h \in \Ia$ by construction. We proceed by showing that $\La \subseteq \Txt\Psd\Ex(h')$. For $L \in \La$, we distinguish the following cases.
  \begin{enumerate}
    \itemin{1. Case:} $L$ is finite. We, again, distinguish multiple cases.
    \begin{enumerate}
      \itemin{1.1. Case:} There exist $k < |L|$ and $t_0$ such that, for all $t > t_0$, $\refuted(L[k], L, t)$ and $L \subseteq W_{q(L[k])}$. Let $k_0$ be minimal such. Then, by lines~\ref{AlgPsdWb_Start_1} to~\ref{AlgPsdWb_Start_3}, $h'$ will converge to $q(L[k_0])$ on any text of $L$. We have $L \subseteq W_{q(L[k])} \subseteq W_{h(L[k], k)}$. 
      As $h$ is target-cautious, we get $L = W_{h(L[k],k)}$ and, thus, $L=W_{q(L[k])}$ as desired.
      \itemin{1.2. Case:} There exists $t_0$ such that, for all $k < |L|$ and $t > t_0$, we have $\refuted(L[k], L, t)$ and $L \not \subseteq W_{q(L[k])}$. In this case, Algorithm~\ref{Algo:PsdWbEx} does not halt before line~\ref{AlgPsdWb_Ind_1}. Then, if $\refuted(L,L,t)$, the algorithm outputs $\ind(L)$ from lines~\ref{AlgPsdWb_Ind_1} and~\ref{AlgPsdWb_Ind_2}. Otherwise, the algorithm does not terminate before line~\ref{AlgPsdWb_knull} and we have $k_0 = |L|$. Then, the algorithm outputs $\ind(L)$ from line~\ref{AlgPsdWb_final_2}. 
      \itemin{1.3. Case:} None of any of the previous cases applies. Then, as Case~1.1 does not apply, the algorithm does not terminate within lines~\ref{AlgPsdWb_Start_1} to~\ref{AlgPsdWb_Start_3}. Furthermore, since additionally Case~1.2 does not apply either, the algorithm does not terminate before line~\ref{AlgPsdWb_knull} and we have that $k_0 < |L|$. Moreover, for any $k$ with $k_0 \leq k \leq |L|$, we have $h(L[k], k) = h(L[k_0], k_0)$ as the hypotheses are not refuted. In particular, $h(L[k_0], k_0)$ is the final hypothesis of $h$, implying $W_{h(L[k_0], k_0)} = L$ and $W_{q(L[k_0])} = L$. Let $s$ be minimal such that there exists $k_1$ with $k_0 \leq k_1 \leq |L|$ and 
      $$\forall k' < k_0 \colon \refuted(L[k'], L[k_1], s) \wedge L[k_1] \not \subseteq W_{q(L[k'])}.$$ 
      Let $k_1$ be minimal such. As $L = W_{q(L[k_0])}$, we have that for all $t$ large enough $L = W_{q(L[k_0])}^t$. For the if-clause in line~\ref{AlgoPsdWb_qCons}, we distinguish the following cases. If $k_1 = |L|$, this shows convergence of $h'$ to $\ind(L[k_1])$. If $k_1 < |L|$, then $h'$ converges to $q(L[k_0])$.
    \end{enumerate}
    \itemin{2. Case:} $L$ is infinite. Let $T$ be the canonical text for $L$. Let $k_T$ be minimal such that $(T[k_T], k_T)$ is a locking information for $h$ on $L$. For $D,t$, with $\content(T[k_T]) \subseteq D \subseteq L$ and $t \geq k_T$, large enough, we have for all $k'<k_T$ that $\refuted(L[k'], D, t)$. Since poisoning either stops the enumeration of data (making $W_{q(L[k'])}$ finite) or some data from $L$ is missing in $W_{q(L[k'])}$, we have $L \not \subseteq W_{q(L[k'])}$. Thus, once $D$ contains enough data, the algorithm does not halt on lines~\ref{AlgPsdWb_Start_1} to~\ref{AlgPsdWb_Start_3}.

    For all $k \geq k_T$ and all $D \subseteq L, t \geq k_T$ large enough, we have $\neg \refuted(L[k], D, t)$, $h(L[k], k) = h(L[k_T], k_T)$ and $W_{q(L[k_T])} = L$. Thus, the algorithm does not halt before line~\ref{AlgPsdWb_knull} with $k_0 = k_T$. Let $s$ be minimal such that there exists $k_1$ with $k_T \leq k_1 \leq s$ and 
    $$\forall k' < k_T \colon \refuted(D[k'], D[k_1], s) \wedge D[k_1] \not \subseteq W_{q(D[k'])}.$$ 
    Let $k_1$ be minimal such. As $L = W_{q(L[k_T])}$ and $L$ is infinite, we have that, for all $t$ large enough, $D[k_1 + 1] \subseteq W_{q(L[k_1])}^t$. Thus, by line~\ref{AlgPsdWb_final_1}, $h'$ converges to $q(L[k_T])$, an index for $L$.
  \end{enumerate}
  Finally, we show that $h'$ is witness-based on $\Txt(\La)$ and $\Txt(\Sa)$. Let $T$ be an according text. For all $i \in \N$, let $D_i = \content(T[i])$. Furthermore, let $k_i$ be such that Algorithm~\ref{Algo:PsdWbEx} on information $(D_i, i)$ returns either $\ind(D_i[k_i])$ or $q(D_i[k_i])$. We first note that, for all $i \in \N$, we have
  \begin{align}
    (W_{h'(D_i, i)})_{\leq \max(D_i[k_i])} = D_i[k_i]. \label{Eq:D_ik_i}
  \end{align}
  Intuitively, $h'(D_i,i)$ is consistent on $D_i[k_i]$ and does not contain any other elements which are smaller than any element thereof. This is immediate if the output is $\ind(D_i[k_i])$. When the output is $q(D_i[k_i])$, the corresponding consistency has to be witnessed, see line~\ref{AlgPsdWb_Start_2} and line~\ref{AlgoPsdWb_qCons}, respectively. Smaller elements cannot be enumerated since $q$ applies function $c$. This proves that Equation~\ref{Eq:D_ik_i} holds.

  Now, let $a,b$ and $d$, with $a < d \leq b$, be such that $h'(D_a,a) \neq h'(D_d, d)$. We show 
  \[
    (D_b \cap W_{h'(D_b, b)}) \setminus W_{h'(D_a,a)} \neq \emptyset.
  \]
  We distinguish the following cases.
  \begin{enumerate}
    \itemin{1. Case:} $D_b[k_b] \subsetneq D_a[k_a]$. It is straightforward to verify that, by Algorithm~\ref{Algo:PsdWbEx}, the only way this is possible is by having $h'(D_b,b) = q(D_b[k_b])$ from line~\ref{AlgPsdWb_final_1}, while, for $k \geq k_b$ and therefore $k \geq k_a$, $D_b[k+1] \subseteq W_{q(D_b[k_b])}$ as desired.
    \itemin{2. Case:} $\exists x \in D_b \setminus D_a[k_a] \, \exists y \in D_a[k_a] \colon x < y$. Let $x$ be minimal such. Using Equation~\eqref{Eq:D_ik_i} it suffices to show that $x \in D_b[k_b]$. Assume the opposite. We get that $D_b[k_b] \subsetneq D_a[k_a]$. Thus, as in the previous case, we get $h'(D_b,b) = q(D_b[k_b])$ from line~\ref{AlgPsdWb_final_1} and $h'(D_a,a) = \ind(D_a[k_a])$ from line~\ref{AlgPsdWb_final_2}. This yields a contradiction, as the hypothesis on $D_b[k_b]$ was rejected when producing the output $h'(D_a, a)$ (as $k_a > k_b$).
    \itemin{3. Case:} $D_a[k_a] = D_b[k_b]$ and $k_a \leq k_b$. Here, we distinguish four cases depending on what line in Algorithm~\ref{Algo:PsdWbEx} led to the output of $h'(D_a,a)$. 
    \begin{enumerate}
      \itemin{3.1. Case:} $h'(D_a,a) = q(D_a[k_a])$ from line~\ref{AlgPsdWb_Start_3}. We consider Algorithm~\ref{Algo:PsdWbEx} on input $(D_b, b)$. If $\refuted(D_a[k_a], D_b, b) \wedge D_b \subseteq W_{q(D_a[k_a])}$, then there was no mind change between $h'(D_a, a)$ and $h'(D_b, b)$, a contradiction to what we assumed. As $\refuted(D_a[k_a],D_a,a)$ and the predicate $\refuted$ is monotone in its second and third component, we get $D_b \not\subseteq W_{q(D_a[k_a])}$. This shows $k_b > k_a$. Thus, if $h'(D_b, b) = \ind(D_b)$ (from line~\ref{AlgPsdWb_Ind_2}), we get $(D_b \cap W_{h(D_b, b)}) \setminus W_{h(D_a, a)} = D_b \setminus W_{h(D_a, a)} \neq \emptyset$ as desired. 

      Otherwise, we get that $k_b$ is such that, with $k_0$ as chosen by Algorithm~\ref{Algo:PsdWbEx} on $(D_b, b)$, 
      \[
        \forall k' < k_0 \colon \refuted(D_b[k'], D_b[k_b], s) \wedge D_b[k_b] \not\subseteq W_{q(D_b[k'])}.
      \]
      Since we assumed a mind change between $h'(D_a, a)$ and $h'(D_b, b)$, we get $k_a < k_0$. Thus, we have $D_b[k_b] \not\subseteq W_{q(D_a[k_a])}$, while $D_b[k_b] \subseteq W_{h'(D_b, b)}$ from Equation~\eqref{Eq:D_ik_i}, as desired.
      \itemin{3.2. Case:} $h'(D_a,a) = \ind(D_a[k_a])$ from line~\ref{AlgPsdWb_Ind_2}. Particularly, it holds that $D_a = D_a[k_a]$. If $D_b = D_a$, there could not be a mind change between $h'(D_a, a)$ and $h'(D_b, b)$. Hence, $D_a \subsetneq D_b$. Now, we also have $k_b > k_a$. By Equation~\eqref{Eq:D_ik_i}, we get $D_b[k_b] \subseteq W_{h'(D_b, b)}$. Altogether, as desired we have
      \[
        \left((D_b \cap W_{h'(D_b, b)}) \setminus W_{h'(D_a, a)} \right) \supseteq \left( D_b[k_b] \setminus D_b[k_a] \right) \neq \emptyset.
      \]
      \itemin{3.3. Case:} $h'(D_a,a) = q(D_a[k_a])$ from line~\ref{AlgPsdWb_final_1}. By assumption, $q(D_a[k_a])$ was abandoned until $(D_b, b)$. Hence, we have $\refuted(D_a[k_a], D_b, b)$ and $D_b \not\subseteq W_{q(D_a[k_a])}$. If $h'(D_b, b)$ comes from lines~\ref{AlgPsdWb_Start_3} or~\ref{AlgPsdWb_Ind_2}, we have $D_b \subseteq W_{h'(D_b, b)}$ as desired. Otherwise, if $h'(D_b, b)$ comes from lines~\ref{AlgPsdWb_final_1} or~\ref{AlgPsdWb_final_2}, by definition, we an output which contains elements not in $W_{q(D_a[k_a])}$ as desired.
      \itemin{3.4. Case:} $h'(D_a,a) = \ind(D_a[k_a])$ from line~\ref{AlgPsdWb_final_2}. This case is almost completely analogous to the previous, except if, for some $k_0$, $h'(D_b, b) = q(D_b[k_0])$ from line~\ref{AlgPsdWb_final_1} and $k_a = k_b$. In this case, we have
      \[
        \left( (D_b \cap W_{h'(D_b, b)}) \setminus W_{h'(D_a, a)} \right) \supseteq \left( D_b[k_b +1] \setminus D_b[k_a] \right) \neq \emptyset.
      \] 
      This is what we desired for. \qedhere
    \end{enumerate}
  \end{enumerate} \noqed
\end{proof}

This result shows that a similar situation as for $\G$-learners holds true for $\Psd$-learners. Admittedly, it is not true that every $\Psd$-learner may be assumed total, as for example is shown by \citet{KS16} who provide an example of such a learner. However, the work of \citet{DoskocK20} shows that, in particular, Gold-style target-cautious learners may be assumed partially set-driven. This result does not suffice for our needs however, as it is only stated for (possibly) partial learners. We extend this result within our framework to fit our needs and obtain a powerful normal form for partially set-driven witness-based learners, see Corollary~\ref{Coro:NF-PsdEx}.

\begin{theorem}
  Let $\Ia$ be an {\Rmonoid} and $\Sa \subseteq \Pow(\N)$ closed under subsets. Furthermore, let $\La$ be a class of languages. Then, the following are equivalent.
  \begin{enumerate}[label=\textnormal{(\arabic*)}]
    \item $\La$ can be $\Txt\G\CautTar\Ex$-learned by an $\Ia$-learner which is defined also on $\Txt(\Sa)$.\label{Glearner}
    \item $\La$ can be $\Txt\Psd\CautTar\Ex$-learned by an $\Ia$-learner which is defined also on $\Txt(\Sa)$.\label{Psdlearner}
  \end{enumerate}
\end{theorem}

\begin{proof}
  We generalize the proof of $[\Txt\G\CautTar\Ex] = [\Txt\Psd\CautTar\Ex]$ as can be found in \citet[Thm.~2]{DoskocK20}. The direction \ref{Psdlearner}$\Rightarrow$\ref{Glearner} is immediate. For the other, let $h \in \Ia$ be a learner which $\Txt\G\CautTar\Ex$-learns $\La$ and is defined on $\Txt(\Sa)$. To ensure correct learning on $\Txt(\La)$, we define a learner $h'$ to search for the minimal, possible locking sequence given a finite set $D$ and $t \geq 0$ as information. This will also maintain $h'$ being defined on $\Txt(\Sa)$. Formally, with $D_\#^{\leq t}$ being the set of all sequences of elements in $D_\# \coloneqq D \cup \{ \# \}$ of at most length $t$, we define $h'$ as
  \begin{align*}
    M_{D,t} &\coloneqq \left\{ \sigma \in D_\#^{\leq t} \mid \forall \tau \in D_\#^{\leq t} \colon h(\sigma) = h(\sigma \tau) \right\}, \\
    h'(D,t) &\coloneqq \begin{cases} h \left(\min(M_{D,t})\right), &\falls M_{D,t} \neq \emptyset, \\ h(\varepsilon), &\sonst \end{cases}
  \end{align*}
  Obviously, if $h$ is defined on $T \in \Txt(\Sa)$ then so is $h'$. Next, we show that $h'$ $\Txt\G\CautTar\Ex$-learns $\La$. Let $L \in \La$ and let $T \in \Txt(L)$. By \citet{BlumBlum75} there exists a locking sequence $\sigma$ for $h$ on $L$. Let $\sigma_0$ be a minimal such locking sequence. Now, let $n_0$ be large enough such that, with $D_0 \coloneqq \content(T[n_0])$ for notational convenience, 
  \begin{itemize}
    \item $\content(\sigma_0) \subseteq D_0$, 
    \item $| \sigma_0 | \leq n_0$ and 
    \item for all $\sigma' < \sigma_0$ there exists $\tau' \in (D_0)_\#^{\leq n_0}$ witnessing $\sigma' \notin M_{D_0, n_0}$. 
  \end{itemize}
  Then, $\min(M_{D_0, n_0}) = \sigma_0$. Thus, for $n \geq n_0$ we have $h'(\content(T[n]),n) = h(\sigma_0)$, and $W_{h'(\content(T[n]),n)} = W_{h(\sigma_0)} = L$. Thus, $L \in \Txt\Psd\Ex(h')$. As $h'$ mimics $h$ on sequences in $L^*$, we also have that $h'$ is $\CautTar$.
\end{proof}

\begin{corollary}\label{Coro:NF-PsdEx}
  We have that $[\tau(\Wb)\Txt\Psd\Ex] = [\Txt\G\CautTar\Ex]$.
\end{corollary}

This result cannot be extended to set-driven learners, as these are known to be weaker than partially set-driven ones, see \citet{KS95} or \citet{DoskocK20}. Another difference is that set-driven learners may be assumed conservative, weakly monotone and cautious without loss of generality, compare the results of \citet{KS95} and \citet{KP16}. Nonetheless, we again provide a general result including witness-based learners.

\begin{theorem}\label{Thm:AllSd}
  Let $\Ia$ be an {\Rmonoid} and $\Sa \subseteq \Pow(\N)$ closed under subsets. Let $\La$ be a class of languages. Then, the following are equivalent.
  \begin{enumerate}[label=\textnormal{(\arabic*)}]
    \item $\La$ can be $\Txt\Sd\Wb\Ex$-learned by an $\Ia$-learner which is $\Wb$ also on $\Txt(\Sa)$.\label{SdWb}
    \item $\La$ can be $\Txt\Sd\Ex$-learned by an $\Ia$-learner which is defined also on $\Txt(\Sa)$.\label{SdT}
  \end{enumerate}
\end{theorem}

\begin{proof}
  The direction \ref{SdWb}$\Rightarrow$\ref{SdT} is immediate. For the other, we follow the proof of $[\Txt\Sd\Ex] = [\Txt\Sd\Conv\Ex]$, see \citet[Thm.~7.1]{KS95}. Let $\La$ be $\Txt\Sd\Ex$-learned by an $\Ia$-learner $h$ which is defined on $\Txt(\Sa)$. By Theorem~\ref{thm:ThmSynDec}, we may assume $h \in \Ia$ to be syntactically decisive on $\Txt(\La)$ and $\Txt(\Sa)$. We show that $\La$ can be $\Txt\Sd\Wb\Ex$-learned by an $\Ia$-learner $h'$ which is $\Wb$ on $\Txt(\Sa)$. 

  In order to define learner $h'$, we need the following auxiliary function. For a set $D$ and $x \in \N$, we write $D_{> x}$ for the set of all elements in $D$ which are larger than $x$. Analogously, we use the notation $D_{< x}$ for the set of all elements in $D$ which are less than $x$. Then, for finite $D \subseteq \N$, we define $p$ as
  \[
    W_{p(D)} = D \cup \bigcup_{t \in \N} 
    \begin{cases} 
      \emptyset, &\falls \neg (D \subseteq W_{h(D)}^t), \\ 
      \emptyset, &\sonstfalls \exists D', D \subseteq D' \subseteq W_{h(D)}^t\colon \\ 
      &\phantom{\text{else}} \bullet h(D') \text{ is not defined after } t \text{ steps, or} \\
      &\phantom{\text{else}} \bullet \exists D'', D \subseteq D'' \subseteq W_{h(D)}^t \colon h(D) = h(D') \neq h(D''), \\
      (W_{h(D)}^t)_{> \max(D)}, &\sonst
    \end{cases}
  \]
  Now, we define $h'$. For finite $D \subseteq \N$ and $k \in \N$, we write $D[k]$ for the set of the $k$ smallest elements in $D$. If $k > |D|$, then $D[k] = D$. For finite $D\subseteq \N$, let $k_D \leq |D|$ be minimal such that, for all $D' \in [D[k_D], D]$, $h(D') = h(D)$. Then, we define
  \[
    h'(D) = p(D[k_D]).
  \]
  Intuitively, $h'$ mimics $h$ on the smallest set (sorted in ascending order) on which no mind change is witnessed. There, it only enumerates larger elements which do not cause a mind change.

  We show that $h'$ $\Txt\Sd\Wb\Ex$-learns $\La$ and is witness-based on $\Txt(\Sa)$. First, we show that $h'$ $\Txt\Sd\Ex$-learns $\La$. Let $L \in \La$ and let $D_0 \subseteq L$ be a locking set for $h$ on $L$. Without loss of generality, we may assume that $(L \setminus D_0)_{< \max(D_0)} = \emptyset$, that is, there are no elements smaller than $\max(D_0)$ which are in $L$ but not in $D_0$. Let $D$ such that $D_0 \subseteq D \subseteq L$. We show that $W_{h'(D)} = L$. First, note that we have that $p(D[k_D]) = p(D_0[k_{D_0}])$. Hence, by definition, we get $h'(D) = p(D_0[k_{D_0}])$. Then, for all $D'$ with $D_0 \subseteq D' \subseteq L$, we have that $h(D')$ is defined and $h(D_0) = h(D')$. Thus, $W_{p(D_0[k_{D_0}])} = L$.
 
  Lastly, we show that $h'$ is $\Wb$ on $\Txt(\La)$ and $\Txt(\Sa)$. Let $T$ be an according text. Let $n_1 < n_2$ and $D_1 = \content(T[n_1])$ and $D_2 = \content(T[n_2])$ with $h'(D_1) \neq h'(D_2)$. Let $n_3 > n_2$ with $D_3 = \content(T[n_3])$. For $i \in \{1,2,3\}$, let $D_i'$ be such that $h'(D_i) = p(D_i')$. We show that 
  \[
    (D_3 \cap W_{p(D_3')}) \setminus W_{p(D_1')} \neq \emptyset.
  \]
  We distinguish the following cases.
  \begin{enumerate}
    \itemin{1. Case:} $D_1' \subseteq D_3'$. Assume $W_{p(D_1')} \supseteq D_3'$. Then, by definition of $p$, for $D'' = D_3'$ we have that $D_1' \subseteq D'' \subseteq W_{p(D_1')}$ and $h(D_1') = h(D'')$, a contradiction.
    \itemin{2. Case:} $D_3' \setminus D_1' \neq \emptyset$, but not $D_1' \subseteq D_3'$. Let $X = D_3' \setminus D_1'$. By definition of $D_3'$, we have that $h(D_3') = h(D_3)$. Note that these are not equal to $h(D_1') = h(D_1)$ as $h$ is syntactically decisive. Thus, as $D_3' = D_1' \cup X$ but $h(D_1') \neq h(D_3')$, we have that $W_{p(D_1')}$ cannot contain all of $D_3'$. 
  \end{enumerate}
  By the search $h'$ conducts upon choosing its hypothesis, the case $D_1' \setminus D_3' \neq \emptyset$ cannot be realized. Thus, the proof is concluded.
\end{proof}

Immediately, we see that any total $\Sd$-learner may be assumed globally witness-based and any (possibly) partial $\Sd$-learner may be assumed witness-based on the languages it learns. In contrast to $\G$- and $\Psd$-learners, we show set-driven learners may not be assumed total in general. We provide a separating class using self-learning classes as presented in \citet{CK16}.

\begin{theorem}
  We have that $[\Txt\Sd\Ex] \setminus [\totalCp\Txt\Sd\Ex] \neq \emptyset$.
\end{theorem}

\begin{proof}
  We show the separation using the Operator Recursion Theorem (\ORT). In order to define the $\Txt\Sd\Ex$-learner $h$, we need the auxiliary predicate $Q$ defined as, for all $e, e', k \in \N$,  
  \begin{align*}
    Q(e, e', k) \Leftrightarrow \forall k' < k\colon \varphi_{e'}(\llangle e, e', k' \rrangle) \neq \varphi_{e'}(\llangle e, e', k' + 1 \rrangle).
  \end{align*}
  For finite $D$, let $k_D = \max(\pi_2(D))$ and
  \begin{align*}
    h(D) = \begin{cases} 
      \ind(\emptyset), &\falls D = \emptyset, \\
      e, &\sonstfalls \pi_0(D) = \{e\} \wedge \pi_1(D) = \{ e' \} \wedge Q(e, e', k_D), \\
      \ind(\llangle e, e', k_D \rrangle), &\sonstfalls \pi_0(D) = \{e\} \wedge \pi_1(D) = \{ e' \} \wedge \neg Q(e, e', k_D), \\
      \divs, &\sonst
    \end{cases}
  \end{align*}
  Intuitively, learner $h$ waits with its decision until it sees what a possible learner $\varphi_{e'}$ does. Then, it outputs its hypothesis accordingly. Let $\La = \Txt\Sd\Ex(h)$ and assume there exists some total learner $h'$ such that $\La \subseteq \totalCp\Txt\Sd\Ex(h')$. Let $e'$ be such that $\varphi_{e'} = h'$. Using \ORT, there exists $e$ such that 
  \begin{align*}
    W_e = \{ \langle e, e', i \rangle : i \in \N \wedge Q(e, e', i) \}.
  \end{align*}
  We show that we can find a language which $h$ learns, but $h'$ cannot. We distinguish the following cases. 
  \begin{enumerate}
    \itemin{1. Case:} $W_e$ is infinite. Then $h$ learns $W_e$ as it, given finite non-empty $D \subseteq W_e$, outputs $e$. That is the correct behaviour. On the other hand, $h'$ cannot learn $W_e$ as it makes infinitely many mind changes on the text $T\colon n \mapsto \langle e, e', n \rangle$.
    \itemin{2. Case:} $W_e$ is finite. Let $k = \max(\pi_2(W_e))$. We show that $h$ learns 
    \begin{align*}
      L_1 &\coloneqq W_e, \\ 
      L_2 &\coloneqq W_e \cup \{ \langle e, e', k + 1 \rangle \}.
    \end{align*}
    For non-empty $D \subseteq L_1$, $h$ outputs $e$. Once $h$ sees $\langle e, e', k + 1 \rangle$ it outputs $\ind(\llangle e, e', k+1 \rrangle)$. Thus, it learns $L_1$ and $L_2$ correctly. However, as $\neg Q(e, e', k +1)$ and $k$ is minimal such, we have 
    \[
      h'(L_1) = \varphi_{e'}(L_1) = \varphi_{e'}(L_2) = h'(L_2).
    \]
    So, $h'$ cannot distinguish between $L_1$ and $L_2$ and, thus, is not able to learn both languages simultaneously. \qedhere
  \end{enumerate} \noqed
\end{proof}

Altogether, we obtain the following normal form for explanatory set-driven witness-based learners. Any (total) set-driven learner may be assumed (globally) witness-based, respectively. However, total set-driven learners lack the learning power of their (possibly) partial counterpart.

\begin{corollary}
The following equalities hold, while the two classes separate.\label{Coro:NF-SdEx}
\begin{align*}
  [\Txt\Sd\Wb\Ex] = [\Txt\Sd\Ex] \text{ and } [\tau(\Wb)\Txt\Sd\Ex] = [\totalCp\Txt\Sd\Ex].
\end{align*}
\end{corollary}

\section{Semantic Witness-based Learning}\label{Sec:SemWb}

Transitioning the results of the previous section to behaviourally correct learners is not immediate since the connection between cautious, weakly monotone and semantically conservative, the semantic counterpart to conservative learning, learners is yet to be discovered in this setting. The work of \citet{DoskocK20} shows that (target-) cautious learners solely rely on the content of the information given, that is, they may be assumed set-driven in general. We show that the same holds true for semantically witness-based learners, the semantic counterpart to witness-based learners. Expanding the findings of \citet{KSS17}, who show that semantically conservative learners may be assumed semantically witness-based, we show the main result of this section, namely that globally semantically witness-based behaviourally correct $\Sd$-learners are as powerful as (possibly partial) Gold-style semantically conservative ones. 

We observe that the mentioned relaxation of constraints works in three ways. Firstly, we may vary the information we give to the learner without forfeiting learning power, that is, full-information learners are equally powerful as set-driven ones. Secondly, we may swap between the requirement of semantically witness-based and semantically conservative learning. Lastly, we may assume these restrictions to hold globally, that is, on arbitrary text. In this section, Theorems~\ref{Thm:tSemConv-SemConv} to~\ref{Thm:tau(Sem.)} provide the proof for the following theorem. 
\begin{theorem}
	We have that $[\tau(\SemWb)\Txt\Sd\Bc] = [\Txt\G\SemConv\Bc]$. \label{Thm:tauSemWbSd-GSemConv}
\end{theorem}

We give an overview on how we proceed to obtain this result. We show the equality from right to left. Firstly, we show that semantically conservative learners may be assumed so globally. Then, we show that Gold-style globally semantically conservative learners maintain the same learning power even when only basing their hypotheses on the content of the information given, that is, they may be assumed set-driven without loss of learning power. This significantly extends the results of \citet{KSS17}, where such learners are shown to be equally powerful when being partially set-driven or set-driven. Then, lastly, in order to ``jump'' from globally semantically conservative learning to its semantically witness-based counterpart, we generalize the result for the non-global case provided by \citet{KSS17}.

It follows the detailed process. We start with a $\Txt\G\SemConv\Bc$-learner $h$. Firstly, we show how to make $h$ \emph{globally} semantically conservative. The idea here is to monitor all possible prior and posterior hypotheses simultaneously. Especially here, Gold-style learners come in handy as they have all information about prior hypotheses at hand. It is important to the learner whether any previous hypothesis is consistent with the current information. If so, the learner, on this hypothesis, simply follows this prior hypothesis' lead. For this to work out, one has to closely monitor what the learner does on future hypotheses. Here, we take advantage of a peculiar property of consistent semantically conservative learners. Such learners need to include all seen data while not overgeneralizing the target language, meaning that finite languages are learnt as soon as all information of it is seen. Thus, it suffices to just consider information without repetition or pause-symbols. Hence, even a check for all possible future hypotheses, which usually is not achievable for Gold-style learners as the crucial correct mind change may come way after all data is seen, is possible in this scenario. The following theorem holds.

\begin{theorem}
  We have that $[\tau(\SemConv)\Txt\G\Bc] = [\Txt\G\SemConv\Bc]$. \label{Thm:tSemConv-SemConv}
\end{theorem}

\begin{proof}
  The inclusion $[\tau(\SemConv)\Txt\G\Bc] \subseteq [\Txt\G\SemConv\Bc]$ follows immediately. For the other, let $h$ be a learner and let $\La = \Txt\G\SemConv\Bc(h)$. Without loss of generality, we may assume that $h$ is consistent, see \citet{KSS17}. We provide a learner $h'$ which $\tau(\SemConv)\Txt\G\Bc$-learns $\La$.

  We do so with the help of an auxiliary $\tau(\SemConv)\Txt\G\Bc$-learner $\hat{h}$, which only operates on sequences without repetitions or pause symbols. For convenience, we omit explicitly mentioning pause symbols. When $h'$ is given a sequence with repetitions, say $(7,1,5,1,4,\#,3,1)$, it mimics $\hat{h}$ given the same sequence without duplicates, that is, $h'(7,1,5,1,4,\#,3,1) = \hat{h}(7,1,5,4,3)$. First, note that this mapping of sequences preserves the $\subseteq$-relation on sequences, thus making $h'$ also a $\tau(\SemConv)$-learner. Furthermore, it suffices to focus on sequences without duplicates. This is the case since consistent, semantically conservative learners cannot change their mind when presented a datum they have already witnessed (or a pause symbol). Thus, $\hat{h}$ will be presented sufficient information for the learning task, which then again is transferred to $h'$. With this in mind, we only consider \textbf{sequences without repetitions (or pause symbols)} for the entirety of this proof. Sequences where duplicates may potentially still occur (for example when looking at the initial sequence of a text) are also replaced as described above. To ease notation, given a set $A$, we write $\Sq(A)$ for the subset of $A_\#^*$ where the sequences do not contain repetitions.
  Now, we define the auxiliary learner $\hat{h}$.

  \begin{algorithm2e}[h]
    \caption{The auxiliary $\tau(\SemConv)$-learner $\hat{h}$ used in the proof of Theorem~\ref{Thm:tSemConv-SemConv}} \label{Algo:SemWbBc}
    \param{$\Txt\G\SemConv$-learner $h$.}
    \input{Finite sequence $\sigma \in \Sq(\N)$.}
    \output{$W_{\hat{h}(\sigma)} = \bigcup_{t \in \N} E_t$.}
    \init{$t' \leftarrow 0$, $E_0 \leftarrow \content(\sigma)$ and, for all $t>0$, $E_t \leftarrow \emptyset$.}
    \For{$t=0$ \KwTo $\infty$}{
      \uIf{$\exists \sigma' \subsetneq \sigma \colon \content(\sigma) \subseteq W_{\hat{h}(\sigma')}^t$\label{AlgoLine:SemWb:BackwardBeg}}{
        $\Sigma' \leftarrow \{ \sigma' \subsetneq \sigma \mid \content(\sigma) \subseteq W_{\hat{h}(\sigma')}^t\}$ \\
        $E_{t+1} \leftarrow E_t \cup \bigcup_{\sigma' \in \Sigma'} W_{\hat{h}(\sigma')}^t$\label{AlgoLine:SemWb:BackwardEnd}
      }\uElseIf{$\forall \sigma' \subsetneq \sigma \colon \content(\sigma) \not\subseteq W_{h(\sigma')}^t$\label{AlgoLine:SemWb:ForwardBeg}}{
        \If{$\forall \tau \in \Sq\left(W_{h(\sigma)}^{t'} \setminus \content(\sigma)\right) \colon \bigcup_{\tau' \in \Sq\left(W_{h(\sigma)}^{t'} \setminus \content(\sigma)\right)} W_{h(\sigma \tau')}^{t'} \subseteq W_{h(\sigma\tau)}^t$\label{AlgoLine:SemWb:ForwardMid}}{
          $E_{t+1} \leftarrow E_t \cup W_{h(\sigma)}^{t'}$ \\
          $t' \leftarrow t' + 1$ 
        }
      \label{AlgoLine:SemWb:ForwardEnd}}\Else{$E_{t+1} \leftarrow E_t$}
    }
  \end{algorithm2e}
  Consider the learner $\hat{h}$ as in Algorithm~\ref{Algo:SemWbBc} with parameter $h$. Given some input $\sigma$, the intuition is the following. Once $\hat{h}$, on any previous sequence $\sigma'$, is consistent with the currently given information $\content(\sigma)$, the learner only enumerates the same as such hypotheses (see lines~\ref{AlgoLine:SemWb:BackwardBeg} to~\ref{AlgoLine:SemWb:BackwardEnd}). While no such hypothesis is found, $\hat{h}$ does a forward search (see lines~\ref{AlgoLine:SemWb:ForwardBeg} to~\ref{AlgoLine:SemWb:ForwardEnd}) and only enumerates elements if all visible future hypotheses also witness these elements. As already discussed, $\hat{h}$ operates only on sequences without repetitions, thus making it possible to check \emph{all} future hypotheses.

  First we show that for any $L \in \La$ and any $T \in \Txt(L)$ we have, for $n \in \N$,
  \begin{align}
    W_{\hat{h}(T[n])} \subseteq W_{h(T[n])}. \label{eq:tSemConvSubs}
  \end{align}
  Note that, while the (infinite) text $T$ may contain duplicates, the (finite) sequence $T[n]$ does not by our assumption. Now, we show Equation~\eqref{eq:tSemConvSubs} by induction on $n$. The case $n=0$ follows immediately. Assume Equation~\eqref{eq:tSemConvSubs} holds up to~$n$. As $\content(T[n+1]) \subseteq W_{h(T[n+1])}$ by consistency of $h$ and as, for $n' \leq n$, $W_{h(T[n'])} = W_{h(T[n+1])}$ whenever $\content(T[n+1]) \subseteq W_{h(T[n'])}$, we get
  \[
     W_{\hat{h}(T[n+1])} \subseteq \bigcup_{\substack{n' \leq n, \\ \content(T[n+1]) \subseteq W_{\hat{h}(T[n'])}}} W_{\hat{h}(T[n'])} \cup W_{h(T[n+1])} \subseteq W_{h(T[n+1])}.
  \]
  The first inclusion follows as the big union contains all previous hypotheses found in the first if-clause (lines~\ref{AlgoLine:SemWb:BackwardBeg} to~\ref{AlgoLine:SemWb:BackwardEnd}) and as $W_{h(T[n+1])}$ contains all elements possibly enumerated by the second if-clause (lines~\ref{AlgoLine:SemWb:ForwardBeg} to~\ref{AlgoLine:SemWb:ForwardEnd}). Note that the latter also contains $\content(T[n+1])$, thus covering the initialization. The second inclusion follows by the induction hypothesis and semantic conservativeness of~$h$. 

  We continue by showing that $\hat{h}$ $\Txt\G\Bc$-learns $\La$. To that end, let $L \in \La$ and $T \in \Txt(L)$. We distinguish the following two cases. 
  \begin{enumerate}
    \itemin{1. Case:} $L$ is finite. Then there exists $n_0$ such that $\content(T[n_0]) = L$. Let $n \geq n_0$. By $\SemConv$ and consistency of $h$, we have $W_{h(T[n])} = L$. Thus, we have $W_{\hat{h}(T[n])} = L$ as 
    \[
      L = \content(T[n]) \subseteq W_{\hat{h}(T[n])} \overset{\eqref{eq:tSemConvSubs}}{\subseteq} W_{h(T[n])} = L.
    \]
    \itemin{2. Case:} $L$ is infinite. Let $n_0$ be minimal such that $W_{h(T[n_0])} = L$. Then, due to $h$ being semantic conservative, $T[n_0]$ is a $\Bc$-locking sequence for $h$ on $L$ and we have 
    \[
      \forall i < n_0 \colon \content(T[n_0]) \not\subseteq W_{h(T[i])}.
    \]
    Thus, elements enumerated by $W_{\hat{h}(T[n_0])}$ cannot be enumerated by the first if-clause (lines~\ref{AlgoLine:SemWb:BackwardBeg} to~\ref{AlgoLine:SemWb:BackwardEnd}) but only by the second one (lines~\ref{AlgoLine:SemWb:ForwardBeg} to~\ref{AlgoLine:SemWb:ForwardEnd}). We show $W_{\hat{h}(T[n_0])} = L$. The $\subseteq$-direction follows immediately from Equation~\eqref{eq:tSemConvSubs}. For the other direction, let $t'$ be the current step of enumeration. As $T[n_0]$ is a $\Bc$-locking sequence, we have, for all $\tau \in \Sq(W_{h(T[n_0])}^{t'})$, 
    $$\bigcup_{\tau' \in \Sq\left(W_{h(T[n_0])}^{t'} \setminus \content(T[n_0]) \right)} W_{h(T[n_0])^\frown \tau'}^{t'} \subseteq W_{h(T[n_0] ^\frown \tau)} = L.$$
    Thus, at some step $t$, $E_{t+1} \leftarrow W_{h(T[n_0])}^{t'}$ and, then, the enumeration continues with $t' \leftarrow t' +1$. In the end we have $L \subseteq W_{\hat{h}(T[n_0])}$ and, altogether, $L = W_{\hat{h}(T[n_0])}$.

    We now show that, for any $n > n_0$, $L = W_{\hat{h}(T[n])}$ holds. At some point $\content(T[n]) \subseteq W_{\hat{h}(T[n_0])}$ will be witnessed. Thus, $W_{\hat{h}(T[n])}$ will enumerate the same as $W_{\hat{h}(T[n_0])} = L$, and it follows that $L \subseteq W_{\hat{h}(T[n])}$. By Equation~\eqref{eq:tSemConvSubs}, $W_{\hat{h}(T[n])}$ will not enumerate more than $W_{h(T[n])} = L$, that is, $W_{\hat{h}(T[n])} \subseteq W_{h(T[n])} = L$, concluding this part of the proof.
  \end{enumerate}

  It remains to be shown that $\hat{h}$ is $\SemConv$ on arbitrary text $T \in \Txt$. The problem is that when a previous hypothesis becomes consistent with information currently given, the learner may have already enumerated incomparable data in its current hypothesis. This is prevented by closely monitoring the time of enumeration, namely by waiting until the enumerated data will certainly not cause such problems. 
  We prove that $\hat{h}$ is $\tau(\SemConv)$ formally. Let $n < n'$ be such that $\content(T[n']) \subseteq W_{\hat{h}(T[n])}$. We show that $W_{\hat{h}(T[n])} = W_{\hat{h}(T[n'])}$ by case distinction.
  \begin{enumerate}
    \item[$\subseteq$:] The inclusion $W_{\hat{h}(T[n])} \subseteq W_{\hat{h}(T[n'])}$ follows immediately, as by assumption $\content(T[n']) \subseteq W_{\hat{h}(T[n])}$, meaning that at some point the first if-clause (see lines~\ref{AlgoLine:SemWb:BackwardBeg} and~\ref{AlgoLine:SemWb:BackwardEnd}) will find $T[n]$ as a candidate and then $W_{\hat{h}(T[n'])}$ will enumerate $W_{\hat{h}(T[n])}$.
    \item[$\supseteq$:] Assume there exists $x \in W_{\hat{h}(T[n'])} \setminus W_{\hat{h}(T[n])}$. Let $x$ be the first such enumerated and let $t_x$ be the step of enumeration with respect to $h(T[n'])$, that is, $x \in W_{h(T[n'])}^{t_x}$ but $x \notin W_{h(T[n'])}^{t_x - 1}$. Furthermore, let $t_\content$ be the step where $\content(T[n']) \subseteq W_{\hat{h}(T[n])}$ is witnessed for the first time. Now, by the definition of $\hat{h}$, we have
    \[
      W_{\hat{h}(T[n'])} \subseteq W_{h(T[n'])}^{t_\content-1} \cup W_{\hat{h}(T[n])},
    \]
    as $W_{\hat{h}(T[n'])}$ enumerates at most $W_{h(T[n'])}^{t_\content-1}$ until it sees the consistent prior hypothesis, namely $\hat{h}(T[n])$. This happens exactly at step $t_\content -1$, at which $W_{\hat{h}(T[n'])}$ stops enumerating elements from $W_{h(T[n'])}^{t_\content-1}$ and continues to follow $W_{\hat{h}(T[n])}$. Now, observe that $t_x < t_\content$ as $x \in W_{\hat{h}(T[n'])}$ but $x \notin W_{\hat{h}(T[n])}$. But then 
    \[
      x \in \bigcup_{\tau' \in \Sq\left( W_{h(T[n])}^{t_\content} \setminus \content(T[n]) \right)} W_{h(T[n] ^\frown \tau')}^{t_\content} \subseteq W_{\hat{h}(T[n])},
    \]
    which must be witnessed in order for $W_{\hat{h}(T[n])}$ to enumerate $\content(T[n'])$ via the second if-clause (lines~\ref{AlgoLine:SemWb:ForwardBeg} to~\ref{AlgoLine:SemWb:ForwardEnd}), that is, to get $\content(T[n']) \subseteq W_{\hat{h}(T[n])}$. This contradicts $x \notin W_{\hat{h}(T[n])}$, concluding the proof. \qedhere 
  \end{enumerate} \noqed
\end{proof}

Applying this result, we may assume $h$ to be globally semantically conservative. Next, we show that even restricting the learners memory does not affect its learning power. In particular, we show that $h$ may be assumed set-driven. This significantly extends the result shown by \citet{KSS17} where (not necessarily globally) semantically conservative partially set-driven learners may be assumed set-driven. While the latter result relies on such learners requirement to include all seen data while not being allowed to overgeneralize, our result originates from another fact. Being semantically conservative and therefore consistent \citep{KSS17} at the same time means that, given any information, whenever the learner $h$ overgeneralizes it may not change its mind on information from this overgeneralization. Thus, we may pretend that the information given came in a certain, for the sake of simplicity say ascending, order. Then, if the learner suggests an element out of this order in its hypothesis, one simply checks whether learner guesses the current information extended by this element on the shortest string of information not containing this element but maintaining the order. If so, both hypotheses are the same and thus both may include this element. Otherwise, the learner skips that element. We provide the rigorous proof. 

\begin{theorem}
  We have that \label{thm:SemConv}
  $$[\tau(\SemConv)\Txt\Sd\Bc] = [\tau(\SemConv)\Txt\Psd\Bc] = [\tau(\SemConv)\Txt\G\Bc].$$
\end{theorem}

\begin{proof}
  We show all three equalities at once. Let $h$ be a learner and $\La = \tau(\SemConv)\Txt\G\Bc(h)$. Without loss of generality, we may assume $h$ to be globally consistent, as shown in \citet{KSS17}. We provide a learner $h'$ such that $\La \subseteq \tau(\SemConv)\Txt\Sd\Bc(h')$. To that end, we introduce the following auxiliary notation which we use throughout this proof. For each $x \in \N$ and each finite set $D \subseteq \N$, let
  \begin{align*}
    d &\coloneqq \max(D), \\
    \sigma_D &\coloneqq \text{canonical sequence of $D$}, \\
    D_{<x} &\coloneqq \{ y \in D \mid y < x \}.
  \end{align*}
  The latter definition can be extended to $\leq, >$ and $\geq$, as well as infinite sets in a natural way. Now, let $h'$ be such that
  \begin{align*}
    W_{h'(D)} &= D \cup \left(W_{h(\sigma_D)}\right)_{>d} \cup \left\{x \in \left(W_{h(\sigma_D)}\right)_{<d} : D \cup \{x\} \subseteq W_{h(\sigma_{(D_{<x})})}\right\}.
  \end{align*}
  Intuitively, $h'(D)$ simulates $h$ assuming it got the information in the canonical order, that is, $h'(D)$ simulates $h(\sigma_D)$. All elements $x \in W_{h(\sigma_D)}$ such that $x > d$ can be enumerated, as any later, consistent hypothesis will do so as well. If $x < d$, then we check whether the learner $h$ given the canonical sequence up to $x$ is consistent with $D \cup \{ x\}$, that is, whether $D \cup \{ x \} \subseteq W_{h(\sigma_{(D_{<x})})}$. If so, we enumerate $x$ as it will be done by the previous hypotheses as well. Note that, for each finite $D \subseteq \N$, we have
  \begin{align}
    W_{h'(D)} \subseteq W_{h(\sigma_D)}. \label{SubsetProp}
  \end{align}

  We proceed by proving that $h'$ $\tau(\SemConv)\Txt\Sd\Bc$-learns $\La$. First, we show the $\Txt\Sd\Bc$-convergence. The idea here is to find a $\Bc$-locking sequence of the canonical text. Doing so ensures that even if elements are shown out of order they will be enumerated as $h$ will not make a mind change and thus the consistency condition will be observed. To that end, let $L \in \La$. We distinguish whether $L$ is finite or not.
  \begin{enumerate}
    \itemin{1. Case:} $L$ is finite. We show that $W_{h'(L)} = L$. By definition of $h'$, we have $L \subseteq W_{h'(L)}$. For the other inclusion, note that as $h$ is consistent and semantically conservative (which in particular implies it being target-cautious), we have that $W_{h(\sigma_L)} = L$. Then, by Equation~\eqref{SubsetProp}, we have that $W_{h'(L)} \subseteq W_{h(\sigma_L)} = L$, concluding this case.
    \itemin{2. Case:} $L$ is infinite. Let $T_c$ be the canonical text of $L$, and let $\sigma_0$ be a $\Bc$-locking sequence for $h$ on $T_c$. Such a $\Bc$-locking sequence exists, as $h$ is strongly $\Bc$-locking, see \citet[Thm.~7]{KSS17}. Let $D_0 \coloneqq \content(\sigma_0)$. For any input $D \subseteq L$ such that $D \supseteq D_0$, we show that $W_{h'(D)} = L$. By Equation~\eqref{SubsetProp}, we get $W_{h'(D)} \subseteq W_{h(\sigma_D)} = L$. To show $L \subseteq W_{h'(D)}$, let $x \in L$. We distinguish the relative position of $x$ and $d$.
    \begin{enumerate}
      \item[$x > d$:] In this case we have $x \in W_{h'(D)}$ by definition of $h'$.
      \item[$x \leq d$:] In this case either $x \in D$ and we immediately get $x \in W_{h'(D)}$, or we have to check whether $D \cup \{x\} \subseteq W_{h(\sigma_{(D_{<x})})}$. Since $\sigma_0$ is an initial segment of the canonical text of $L$, it holds that $x > \max(\content(\sigma_0))$ and, thus, we get $\sigma_0 \subseteq \sigma_{(D_{<x})}$. Now $W_{h(\sigma_{(D_{<x})})} = L$, meaning that $D \cup \{x\} \subseteq W_{h(\sigma_{(D_{<x})})}$ will be observed at some point in the computation. Thus, $x \in W_{h'(D)}$ in this case as well.
    \end{enumerate}
  \end{enumerate}
  Altogether, we get $W_{h'(D)} = L$ and thus $\Txt\Sd\Bc$-convergence. It remains to be shown that $h'$ is $\tau(\SemConv)$. Let $D' \subseteq D''$ and $D'' \subseteq W_{h'(D')}$. The trick here is that upon checking for consistency with elements shown out of order, the learner has to check the same, minimal sequence regardless whether the input is $D'$ or $D''$. We proceed with the formal proof. Therefore, we expand the initially introduced notation of this proof. For any $x \in \N$ define
  \begin{align*}
    \sigma' &\coloneqq \sigma_{D'}, \\
    d' &\coloneqq \max(D'), \\
    \sigma'_{<x} &\coloneqq \sigma_{(D'_{<x})}.
  \end{align*}
  Analogously, we use $\sigma'', d''$ and $\sigma''_{<x}$ when $D''$ is the underlying set. First, we show that $W_{h(\sigma')} = W_{h(\sigma'')}$. Since $W_{h'(D')}$ enumerates $D''$, that is, $D'' \subseteq W_{h'(D')}$, we have for all $y \in (D'' \setminus D')_{<d'}$ that $D' \cup \{ y\} \subseteq W_{h(\sigma'_{<y})}$ by definition of $h'$. Thus, we have
  \begin{align}
    W_{h(\sigma'_{<y})} = W_{h(\sigma')}. \label{EqualSubseq}
  \end{align}
  Note that, if $(D'' \setminus D')_{<d'}$ is empty, then $\sigma'_{<d'+1} = \sigma'$. Thus, Equation~\eqref{EqualSubseq} also holds true for
  $$m \coloneqq \begin{cases} \min(D''_{<d'} \setminus D'), &\falls D''_{<d'} \setminus D' \neq \emptyset, \\ d' +1, &\sonst \end{cases}$$ 
  Furthermore, it holds true that for any $x \leq m$ we have 
  \begin{align}
    \sigma'_{<x} = \sigma''_{<x}. \label{Eq:SeqSmall}
  \end{align}
  By Equations~\eqref{SubsetProp} and \eqref{EqualSubseq}, we have $D'' \subseteq W_{h'(D')} \subseteq W_{h(\sigma')} = W_{h(\sigma'_{<m})}$. As, by Equation~\eqref{Eq:SeqSmall}, $\sigma'_{<m} = \sigma''_{<m} \subseteq \sigma''$ and $h$ is $\tau(\SemConv)$, we get
  \begin{align}
    W_{h(\sigma')} = W_{h(\sigma'')}. \label{SigmaEquiv}
  \end{align}

  We conclude the proof by showing that $W_{h'(D')} = W_{h'(D'')}$. We check each direction separately by checking every possible position of an element, which is a candidate for enumeration, relative to the given information $D'$ and $D''$. 
  \begin{enumerate}
    \item[$\supseteq$:] Let $x \in W_{h'(D'')}$. For $x \in D''$ we have $x \in W_{h'(D')}$ by assumption. Otherwise, by Equations~\eqref{SubsetProp} and~\eqref{SigmaEquiv}, we get $x \in W_{h(\sigma')}$. Thus, $x$ will be considered in the enumeration of $W_{h'(D')}$. We distinguish the relation between $x$ and $d'$. 
    \begin{enumerate}
      \item[$x > d'$:] In this case $x \in (W_{h(\sigma')})_{>d'} \subseteq W_{h'(D')}$. 
      \item[$x < d'$:] As $d' \leq d''$ and since $x$ is enumerated into $W_{h'(D'')}$, we have $D'' \cup \{ x \} \subseteq W_{h(\sigma''_{<x})}$. We, again, distinguish the relative position of $x$ and $m$ and get
      \begin{enumerate}
        \item[$x < m$:] $D' \cup \{ x \} \subseteq D'' \cup \{ x \} \subseteq W_{h(\sigma''_{<x})} \overset{\eqref{Eq:SeqSmall}}{=} W_{h(\sigma'_{<x})}$,
        \item[$m < x < d'$:] $D' \cup \{ x \} \subseteq D'' \cup \{ x \} \subseteq W_{h(\sigma''_{<x})} \overset{(\ast)}{=} W_{h(\sigma'')} \overset{\eqref{SigmaEquiv}}{=} W_{h(\sigma')} \overset{\eqref{EqualSubseq}}{=} W_{h(\sigma'_{< m})} \overset{(\ast)}{=} W_{h(\sigma'_{<x})}.$
      \end{enumerate}
      We use $h$ being $\tau(\SemConv)$ in the steps marked by $(\ast)$. Thus, $x \in W_{h'(D')}$. 
    \end{enumerate}
    \item[$\subseteq$:] Let $x \in W_{h'(D')}$. For $x \in D''$ we have $x \in W_{h'(D'')}$ by definition of $h'$. Otherwise, note that
    \begin{align*}
      x \in D'' \cup \{ x \} \subseteq W_{h'(D')} \subseteq W_{h(\sigma')} \overset{\eqref{SigmaEquiv}}{=} W_{h(\sigma'')}.
    \end{align*}
    Thus, $x$ will be considered in the enumeration of $W_{h'(D'')}$. We now distinguish between the possible relation of $x$ and $d''$.
    \begin{enumerate}
      \item[$x > d''$:] In this case $x \in W_{h'(D'')}$ by definition of $h'$.
      \item[$x < d''$:] Here, we show that $D'' \cup \{x\} \subseteq W_{h(\sigma''_{<x})}$ is witnessed and, thus, $x$ is enumerated by $W_{h'(D'')}$. We distinguish the following cases.
        \begin{align*}
          x < m &: D'' \cup \{ x \} \subseteq W_{h(\sigma'_{<x})} \overset{\eqref{Eq:SeqSmall}}{=} W_{h(\sigma''_{<x})}, \\
          m < x < d' &: D'' \cup \{ x \} \subseteq W_{h(\sigma'_{<x})} \overset{(\ast)}{=} W_{h(\sigma'_{<m})} \overset{\eqref{Eq:SeqSmall}}{=} W_{h(\sigma''_{<m})} \overset{(\ast)}{=} W_{h(\sigma''_{<x})}, \\
          d' < x < d'' &: D'' \cup \{ x \} \subseteq W_{h(\sigma')} = W_{h(\sigma'_{<m})} = W_{h(\sigma''_{<m})} \overset{\eqref{Eq:SeqSmall}}{=} W_{h(\sigma''_{<x})}.
        \end{align*}
        We use $h$ being $\tau(\SemConv)$ in the steps marked by $(\ast)$. In the end, $x \in W_{h'(D'')}$. \qedhere
    \end{enumerate}
  \end{enumerate}\noqed
\end{proof}

Thus, we may assume $h$ to be globally semantically conservative set-driven. Lastly, by extending the result of \citet{KSS17}, who show that $\SemWb$- and $\SemConv$-learners are equally powerful for all considered interaction operators, to hold for the global counterpart as well, we see that $h$ remains equally powerful even when being globally semantically witness-based. The following result concludes the proof of Theorem~\ref{Thm:tauSemWbSd-GSemConv} and, therefore, also this section.

\begin{theorem}
  For $\beta \in \{ \G, \Psd, \Sd \}$, we have that\label{Thm:tau(Sem.)}
  $$[\tau(\SemWb)\Txt\beta\Bc] = [\tau(\SemConv)\Txt\beta\Bc].$$
\end{theorem}

\begin{proof}
  \citet[Thm. 8]{KSS17} show that $\SemWb$ and $\SemConv$ allow for consistent $\Bc$-learning. This also holds true when the restrictions are global, thus, $[\tau(\Cons\SemWb)\Txt\beta\Bc] = [\tau(\SemWb)\Txt\beta\Bc]$ and $[\tau(\Cons\SemConv)\Txt\beta\Bc] = [\tau(\SemConv)\Txt\beta\Bc]$. Since $\Cons \cap \SemWb = \Cons \cap \SemConv$, as shown by \citet[Lem. 11]{KSS17}, the theorem holds.
\end{proof}

\section{Conclusion and Future Work}\label{Sec:Concl}

In this work, we study the behaviour of witness-based learners in different settings. Being a specialization to important restrictions within inductive inference, first studies of witness-based learners have been provided by \citet{KS16} and \citet{KSS17}. With this work, we provide a thorough investigation of the behaviour of these learners and provide \emph{normal forms} thereof. In particular, we provide a general framework with which we obtain results for witness-based explanatory learners with multiple additional restrictions. Most notably, we show that globally witness-based $\Psd$-learners are equally powerful as target-cautious $\G$-learners. Furthermore, we provide results in the behaviourally correct case, showing that globally semantically witness-based set-driven learners are equally powerful as semantically conservative full-information learners. The latter result is vital in order to obtain a \emph{full map for delayable restrictions in the behaviourally correct} case, which is left as future work.

\acks{%
  This work was supported by DFG Grant Number KO 4635/1-1. 
}

\bibliography{LTbib}

\end{document}